\documentclass[conference]{IEEEtran}

\usepackage[utf8]{inputenc}
\usepackage{amsmath}
\usepackage{amsthm}
\usepackage{amssymb}
\usepackage{tikz}
\usepackage[noend]{algpseudocode}
\usepackage[ruled]{algorithm}
\usepackage{todonotes}
\usepackage{multirow}
\usepackage[export]{adjustbox}
\usepackage[multiple]{footmisc}
\usepackage{url}

\newtheorem{theorem}{Theorem}
\newtheorem{lemma}{Lemma}

\theoremstyle{definition}
\newtheorem{definition}{Definition}

\newcommand{\feats}{\mathcal{X}}
\newcommand{\R}{\mathbb{R}}
\newcommand{\N}{\mathbb{N}}
\newcommand{\labels}{\mathcal{Y}}
\newcommand{\dtrain}{\mathcal{D}_{\textit{train}}}
\newcommand{\dtest}{\mathcal{D}_{\textit{test}}}

\newcommand{\iatk}{I^{\textit{atk}}}
\newcommand{\pre}{\textit{pre}}
\newcommand{\post}{\textit{post}}
\newcommand{\paid}{\textit{cost}}
\newcommand{\ipre}{I^{\pre}}
\newcommand{\ipost}{I^{\post}}
\newcommand{\jpre}{J^{\textit{pre}}}
\newcommand{\jpost}{J^{\textit{post}}}
\newcommand{\symatk}[3]{\langle #1 \rangle \rhd \langle #2 \rangle_{#3}}
\newcommand{\sym}{\textit{sym}}

\newcommand{\etal}{\textit{et al.}}

\newcommand{\empirical}[1]{#1}
\newcommand{\revise}[1]{#1}
\newcommand{\revisee}[1]{#1}
\newcommand{\cla}[1]{#1}

\title{Beyond Robustness: Resilience Verification of Tree-Based Classifiers}

\author{
    \IEEEauthorblockN{
    Stefano Calzavara\IEEEauthorrefmark{1}\textsuperscript{1},
    Lorenzo Cazzaro\IEEEauthorrefmark{1}\textsuperscript{1},
    Claudio Lucchese\IEEEauthorrefmark{1}\textsuperscript{1},
    Federico Marcuzzi\IEEEauthorrefmark{1}\textsuperscript{1} and
    Salvatore Orlando\IEEEauthorrefmark{1}\textsuperscript{1}}
    \IEEEauthorblockA{\IEEEauthorrefmark{1}Department of Environmental Sciences, Informatics and Statistics,\\ Ca’ Foscari University of Venice, Italy.\\
    Email: \{name.surname\}@unive.it\\
    }
}

\begin{document}

\maketitle

\begingroup\renewcommand\thefootnote{1}
\footnotetext{Equal contribution}

\begin{abstract}
In this paper we criticize the robustness measure traditionally employed to assess the performance of machine learning models deployed in adversarial settings. To mitigate the limitations of robustness, we introduce a new measure called resilience and we focus on its verification. In particular, we discuss how resilience can be verified by combining a traditional robustness verification technique with a data-independent stability analysis, which identifies a subset of the feature space where the model does not change its predictions despite adversarial manipulations. We then introduce a formally sound data-independent stability analysis for decision trees and decision tree ensembles, which we experimentally assess on public datasets and we leverage for resilience verification. Our results show that resilience verification is useful and feasible in practice, yielding a more reliable security assessment of both standard and robust decision tree models.
\end{abstract}

\section{Introduction}
Machine Learning (ML) is becoming more and more popular nowadays, in particular for classification tasks, yet it is acknowledged to be susceptible to different types of attacks. A number of research papers showed that classifiers trained using standard ML algorithms cannot be deployed in security-sensitive settings, because they are easily fooled in practice and their performance undergoes significant downgrade when their inputs are subject to adversarial manipulations, imperceptible to human experts~\cite{SzegedyZSBEGF13,GoodfellowSS14}. This motivated the development of new performance measures like \emph{robustness}, which generalize traditional measures like accuracy to account for the threats of adversarial manipulations at test time~\cite{MadryMSTV18,RanzatoZ20}. Specifically, given an input $\vec{x}$ and its correct class $y$, robustness requires the classifier to predict the class $y$ also for all the adversarial manipulations $A(\vec{x})$, rather than just for the original input $\vec{x}$. 

Robustness is certainly an intuitive and desirable property to estimate the performance of classifiers deployed in adversarial settings, yet it is sub-optimal because it is strongly dependent on the choice of a specific input $\vec{x}$. While the performance of classifiers must indeed be empirically estimated on a set of correctly labeled inputs (test set), such inputs are normally assumed to be sampled from an underlying data distribution and robustness tells nothing about unsampled data. In other words, a robustness proof for $\vec{x}$ does not provide any guarantee about any other input $\vec{z}$ close to $\vec{x}$ which could have been sampled in place of it. This is concerning, because the actual inputs of the classifier at test time will be different samples drawn from the same distribution of $\vec{x}$, which are not covered by standard security assessments based on robustness. \revise{This problem has been independently acknowledged in very recent work on \emph{global robustness}~\cite{Chen0QLJW21,LeinoWF21}, which advocates the need for verification techniques establishing robustness guarantees on all the possible inputs provided to the classifier. Our work broadly falls in the same research line and shares similarities with existing efforts (cf. Section~\ref{sec:related} for a comparison), yet it takes a different direction.}

\revise{In particular, we here propose a generalization of robustness, called \emph{resilience}, designed to make the security assessment of classifiers more reliable. Resilience generalizes traditional robustness guarantees from a specific test set to all the other possible test sets which could have been sampled in place of it, i.e., which are close to it given an appropriate definition of neighborhood. Resilience thus provides a more conservative account of the security of classifiers than robustness, while retaining its intuitive flavour. Most importantly, the connection between resilience and robustness allows one to leverage traditional tools for robustness verification as the first step of a resilience verification pipeline, thus integrating with significant research efforts spent on robustness verification.}

\paragraph*{Contributions}
In the present paper we make the following contributions:
\begin{enumerate}
    \item We criticize the traditional robustness measure used to estimate the security of classifiers against evasion attacks and we propose an improved measure called resilience. We then discuss how resilience can be estimated by combining an arbitrary robustness verification technique with a \emph{data-independent stability analysis}, which identifies a subset of the feature space where the classifier does not change its predictions despite adversarial manipulations at test time. \revise{The analysis is data-independent because it is based on the classifier alone, rather than on a specific test set (contrary to robustness). We finally present a simple technique to turn any classifier into a globally robust classifier, in the sense of~\cite{LeinoWF21}, by leveraging such data-independent stability analysis} (Section~\ref{sec:resilience}).
    
    \item We propose a data-independent stability analysis for decision trees and decision tree ensembles, a popular class of ML models~\cite{quinlan1986induction}. \revise{The stability analysis is based on \emph{symbolic attacks}, i.e., symbolic representations of a set of instances along with their (relevant) adversarial manipulations, which support the analysis of tree-based classifiers independently of a specific test set.} Our analysis is proved sound and can be readily leveraged to establish both robustness and resilience proofs for tree-based classifiers (Section~\ref{sec:analysis}).
    
    \item We implement our data-independent stability analysis\footnote{We will release our analyzer as open-source upon paper acceptance.} and we experimentally assess its effectiveness on public datasets, by estimating the robustness and resilience of both standard and robust tree models trained using a state-of-the-art adversarial ML algorithm (Section~\ref{sec:experiments}).
\end{enumerate}

Our experimental evaluation shows that resilience verification is both useful and feasible in practice, yielding a more reliable security assessment of classifiers deployed in adversarial settings. \revise{In particular, our experiments show that robustness can be significantly affected by the choice of a specific test set, hence it may give a false sense of security, while resilience is effective at discriminating between secure models and models which turned out to be robust just \emph{by accident}, i.e., thanks to a lucky, specific sampling of the test set. We thus recommend the use of resilience for the security verification of ML models deployed in adversarial scenarios.}

\section{Background}
We introduce here the key technical ingredients required to appreciate this work.

\subsection{Decision Trees for Classification}
Let $\feats \subseteq \R^d$ be a $d$-dimensional vector space of real-valued \textit{features}. An \emph{instance} $\vec{x} \in \feats$ is a $d$-dimensional feature vector $\langle x_1, x_2, \ldots, x_d \rangle$ representing an object in the vector space $\feats$. Each instance is assigned a class label $y \in \labels$ by an unknown \emph{target} function $f: \feats \mapsto \labels$. % For simplicity, we assume $\labels = \{-1, +1\}$ since multi-class classification can be reduced to the binary case, e.g., using a one-vs-rest strategy.

Supervised learning algorithms automatically learn a \emph{classifier} $g: \feats \mapsto \labels$ from a \emph{training set} of correctly labeled instances $\dtrain = \{(\vec{x}_i,f(\vec{x}_i))\}_i$, with the goal of approximating the target function $f$ as accurately as possible. The performance of classifiers is estimated on a \emph{test set} of correctly labeled instances $\dtest = \{(\vec{z}_i,f(\vec{z}_i))\}_i$, normally disjoint from the training set. For example, the \emph{accuracy} measure $a$ counts the percentage of correct predictions out of all the predictions performed on the test set:
\[
a(g,\dtest) = \dfrac{|\{(\vec{z}_i,f(\vec{z}_i)) \in \dtest ~|~ g(\vec{z}_i) = f(\vec{z}_i)\}|}{|\dtest|}
\]

Classifiers can take different shapes, drawn from different sets of hypotheses: in this work, we focus on traditional \emph{binary decision trees}, which are the most common and popular version of such models.

Decision trees can be inductively defined as follows: a decision tree $t$ is either a leaf $\lambda(y)$ for some label $y \in \labels$ or an internal node $\sigma(f,v,t_l,t_r)$, where $f \in \{1,\ldots,d\}$ identifies a feature, $v \in \R$ is a threshold for the feature, and $t_l,t_r$ are decision trees (left and right respectively). At test time, the instance $\vec{x}$ traverses the tree $t$ until it reaches a leaf $\lambda(y)$, which returns the prediction $y$, denoted by $t(\vec{x}) = y$. Specifically, for each traversed tree node $\sigma(f,v,t_l,t_r)$, $\vec{x}$ falls into the left sub-tree $t_l$ if $x_f \leq v$, and into the right sub-tree $t_r$ otherwise. Figure~\ref{fig:tree} represents an example decision tree of depth 2, which assigns the label $+1$ to the instance $\langle 12,7 \rangle$ and the label $-1$ to the instance $\langle 8,6 \rangle$. 

Decision trees are normally combined into an \emph{ensemble} $T = \{t_1,\ldots,t_n\}$: in this case, the ensemble prediction $T(\vec{x})$ is computed by combining together the individual tree predictions $t_i(\vec{x})$, e.g., by performing majority voting on the individually predicted classes. % In this paper we assume the use of the majority voting strategy, which is common for ensembles trained using the popular Random Forest algorithm, leaving other strategies designed for regression trees to future work.

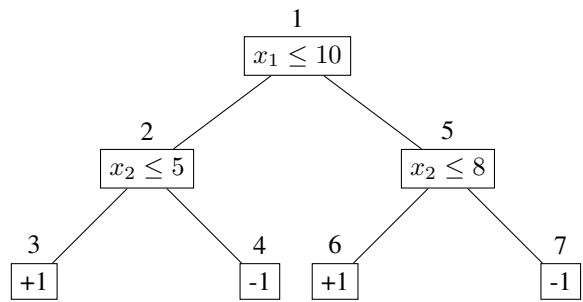
\begin{figure}[t]
\centering
\begin{tikzpicture}[level 1/.style={sibling distance=4cm},level 2/.style={sibling distance=3cm}]
\tikzstyle{every node}=[rectangle,draw]
\node[label=1]{$x_1 \leq 10$}
	child { node[label=2] {$x_2 \leq 5$}
	        child { node[label=3] {+1}}
	        child { node[label=4] {-1}} }
	child { node[label=5] {$x_2 \leq 8$}
		    child { node[label=6] {+1}}
	        child { node[label=7] {-1}} }
;
\end{tikzpicture}
\caption{Example of decision tree}
\label{fig:tree}
\end{figure}

\subsection{Stability and Robustness}
Stability and robustness are classic definitions considered in prior literature on adversarial ML~\cite{RanzatoZ20}. They are used to reason about the security of classifiers against \emph{evasion attacks}, i.e., malicious manipulations of instances at test time aimed at forcing mispredictions. For example, an evasion attack might slightly modify malware to make it look like benign software to the classifier. We write $A(\vec{x})$ to represent the set of all the adversarial manipulations of the instance $\vec{x}$, corresponding to the possible evasion attack attempts against $\vec{x}$. Stability requires the classifier to make the same prediction for all the possible evasion attack attempts.

\begin{definition}[Stability]
The classifier $g$ is \emph{stable} on the instance $\vec{x}$ if and only if, for every adversarial manipulation $\vec{z} \in A(\vec{x})$, we have $g(\vec{z}) = g(\vec{x})$.
\end{definition}

Stability is useful for the security certification of classifiers: if the classifier $g$ is stable on the instance $\vec{x}$, no adversarial manipulation $\vec{z} \in A(\vec{x})$ can be assigned a label different from the classifier prediction $g(\vec{x})$, hence no evasion attack is possible. However, stability does not capture whether a classifier is useful in practice: for example, a trivial classifier which always predicts a constant class is stable on all instances. Hence, the actual property of interest for classifiers deployed in adversarial settings is robustness, which additionally requires the classifier to perform correct predictions.

\begin{definition}[Robustness]
The classifier $g$ is \emph{robust} on the instance $\vec{x}$ if and only if $g(\vec{x}) = f(\vec{x})$ and $g$ is stable on $\vec{x}$.
\end{definition}

Like accuracy, robustness is also normally quantified over a test set $\dtest = \{(\vec{z}_i,f(\vec{z}_i))\}_i$. Formally, we define the robustness measure $r$ as follows:
\[
r(g,\dtest) = \dfrac{|\{(\vec{z}_i,f(\vec{z}_i)) \in \dtest ~|~ g \textnormal{ is robust on } \vec{z}_i\}|}{|\dtest|}
\]

Note that robustness represents a lower bound to accuracy, i.e., $r(g,\dtest) \leq a(g,\dtest)$ for every $g$ and $\dtest$.

\section{Resilience}
\label{sec:resilience}

We now discuss important shortcomings in the traditional robustness measure and we propose a generalization of robustness, called \emph{resilience}, which is designed to mitigate those. We then explain how resilience can be verified in practice and we further elaborate on its design by discussing its connections with a recent definition of global robustness~\cite{LeinoWF21}.

\begin{figure*}[t]
    \centering
    \adjincludegraphics[width=16cm,trim={0 150 0 0},clip]{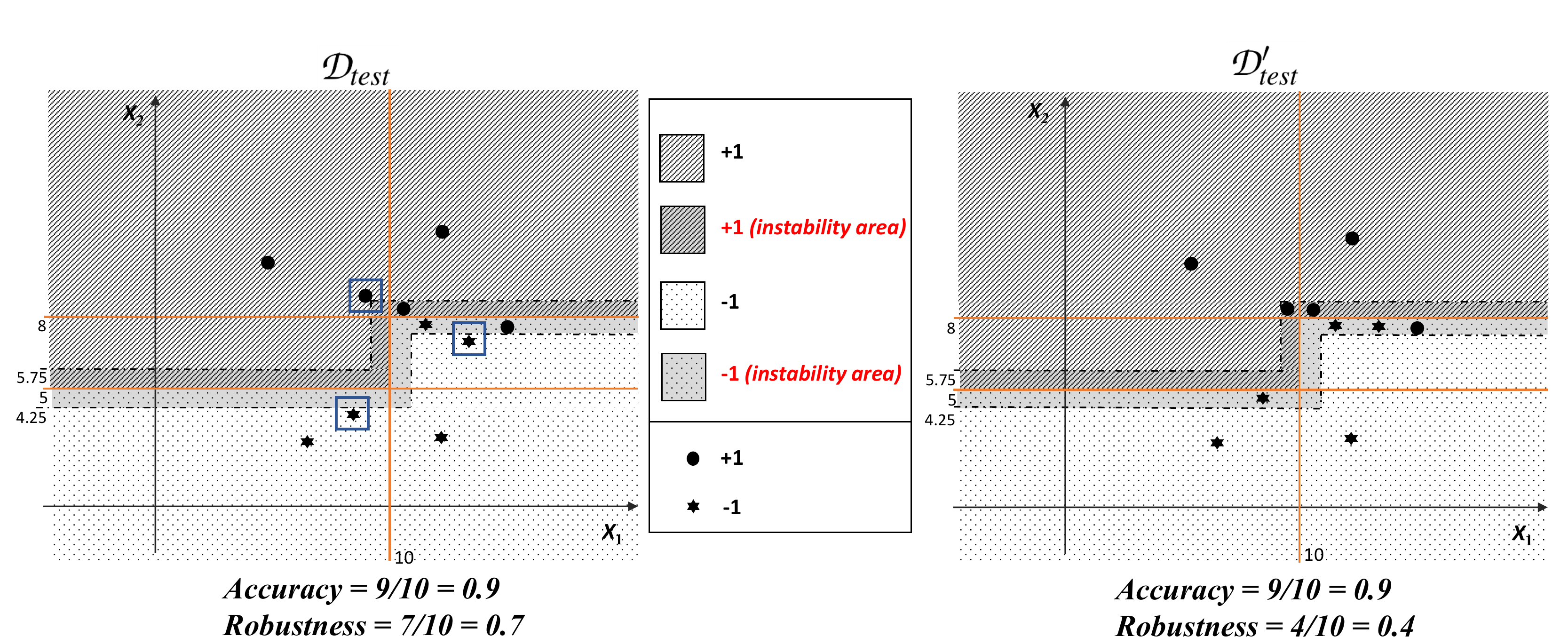}
    \caption{Robustness is not robust: slightly different test sets lead to very different values of robustness}
    \label{fig:robustness}
\end{figure*}

\subsection{Robustness vs. Resilience}
A key problem of robustness is its strong \emph{data-dependence}, i.e., robustness is quantified on a specific test set $\dtest$. Hence, it is possible that even tiny differences between two test sets $\dtest,\dtest'$ might lead to quite different values of robustness. This might give a false sense of security, because a classifier having 0.7 robustness on $\dtest$ might only have 0.4 robustness on $\dtest'$, although both $\dtest$ and $\dtest'$ are representative of the \revisee{same data} distribution of the test instances. We show this problem of robustness in Figure~\ref{fig:robustness} for the same tree classifier of Figure~\ref{fig:tree}, in a two-dimensional space. We assume here an $L_1$-norm attacker such that $A(\vec{x}) = \{\vec{z} \in \feats ~|~ \Vert \vec{z} - \vec{x} \Vert_1 \leq 0.75\}$, leading to the highlighted \textit{instability areas} around the decision boundaries of the tree. \revise{The figure shows that the same classifier provides very different robustness measures on the two close test sets $\dtest$ and $\dtest'$, because more instances of $\dtest'$ fall in the instability area. Hence, the adoption of $\dtest$ over $\dtest'$ for robustness computation might give a false sense of security.}

To mitigate this problem of robustness, we propose \emph{resilience}, a new security measure that explicitly assumes that test instances are sampled from a given data distribution, hence each instance $\vec{x}$ is just a possible sample drawn from a set of neighbours $N(\vec{x})$. 
%We do not assume any specific property of $N$, except the obvious property that for all $\vec{x} \in \feats$ we have $\vec{x} \in N(\vec{x})$. 
\revise{For example, $N(\vec{x})$ may contain all the instances which are within a maximum $L_{\infty}$-distance from $\vec{x}$, as we assume in our experiments.} 
\revisee{In the example of Figure~\ref{fig:robustness}, we define $N(\vec{x}) = \{\vec{z} \in \feats ~|~ \Vert \vec{z} - \vec{x} \Vert_\infty \leq 0.50\}$ and the neighborhood of $\vec{x}$ is thus graphically represented by a small box around $\vec{x}$. Indeed, the figure shows these boxes only for the instances of $\dtest$ whose neighborhood $N(\vec{x})$ intersects the instability areas of the tree and, therefore, other instances in their neighborhood might suffer from evasion attacks.}

Resilience avoids the shortcomings of robustness illustrated in Figure~\ref{fig:robustness}, as it generalizes the idea of robustness to all the test sets which could have been sampled within neighborhoods of the original test set. Formally, resilience requires the classifier to be robust on test instances, while remaining stable on their neighborhoods, for which the correct class labels are unknown.

\begin{definition}[Resilience]
The classifier $g$ is \emph{resilient} on the instance $\vec{x}$ if and only if $g$ is robust on $\vec{x}$ and $g$ is stable on all the instances $\vec{z} \in N(\vec{x})$.
\end{definition}

Back to our example, the resilience of both the test sets of Figure~\ref{fig:robustness} turns out to be $0.4$, i.e., the measured robustness of $\dtest'$, the unlucky test set shown in the right part of the figure. 

Note that resilience generalizes beyond the test set by means of the $N(\vec{x})$ component, which extends the stability guarantees of robustness to an uncountable set of neighbours not included in the test set. Still, similarly to robustness, we can quantify resilience on a test set of correctly labelled instances like in traditional ML pipelines by defining a resilience measure $R$ as follows:
\[
R(g,\dtest) = \dfrac{|\{(\vec{z}_i,f(\vec{z}_i)) \in \dtest ~|~ g \textnormal{ is resilient on } \vec{z}\}|}{|\dtest|}
\]

Observe that resilience provides a lower bound to robustness as claimed, i.e., $R(g,\dtest) \leq r(g,\dtest)$ for every $g$ and $\dtest$.

\subsection{Resilience Verification}
\label{sec:resilience-verification}
Traditional robustness verification approaches cannot be readily applied to prove resilience. \revise{In particular, note that robustness verification takes as input an instance $\vec{x}$ and attempts to assess its stability, while resilience verification requires proving stability for an uncountable set of instances $N(\vec{x})$.} 

Nevertheless, it is possible to estimate resilience by combining robustness verification with a \emph{data-independent} stability analysis. In particular, assume one has a technique to identify the set $X_S = \{\vec{x} \in \feats ~|~ g \textnormal{ is stable on } \vec{x}\}$, then $g$ is resilient on $\vec{x}$ if and only if $g$ is robust on $\vec{x}$ and $N(\vec{x}) \subseteq X_S$.
\revisee{Note that $X_S$ in  Figure~\ref{fig:robustness} corresponds to the whole area except the instability ones.}
Hence, resilience verification reduces to robustness verification with the additional condition $N(\vec{x}) \subseteq X_S$\revisee{, i.e., we also have to check that $N(\vec{x})$ does not intersect the instability area}. \revise{This allows one to take advantage of existing robustness verifiers also to quantify resilience, provided that the stable part of the feature space $X_S$ has been computed.}
\cla{Note that computing $X_S$ is not trivial, in particular for tree ensembles, because a given instance traverses the ensemble reaching a set of leaves, and any of such reachable set of leaves corresponds to a sub-space which should be evaluated for inclusion in $X_S$. Clearly, the number of these sub-spaces grows exponentially with the number of trees in the ensemble. In order to compute $X_S$, and therefore compute resilience, we resort to an approximated approach: computing a tractable under-approximation of $X_S$ for decision trees and decision tree ensembles is a key contribution of the present paper.}

%The reasons why we focus on decision trees and decision tree ensembles are twofold. First, these models operate by means of thresholds, hence performing a data-independent stability analysis is more natural than for other models like deep neural networks. Second, there already exist several effective robustness verification frameworks for these classes of models, e.g.,~\cite{ChenZS0BH19,RanzatoZ20,CalzavaraFL20}, hence the proposed resilience verification technique may yield immediate advantages thanks to prior research.

In particular, the static analysis in Section~\ref{sec:analysis} allows one to compute a set $X_S' \subseteq X_S$, i.e., a sound under-approximation of the portion of the feature space where a tree-based classifier is stable. \cla{The proposed computation of $X_S'$ depends on the classifier alone and not on the test data at hand.}
\cla{In fact, we can exploit the generality of $X_S'$ to efficiently compute lower bounds for the robustness and resilience measures as follows.}

\cla{Given that if $\vec{x} \in X_S'$ and $g(\vec{x}) = f(\vec{x})$, then $g$ is robust on $\vec{x}$, we can define a lower bound $\hat{r}$ on robustness as:}
\[
\hat{r}(g,\dtest) = \dfrac{|\{(\vec{z}_i,f(\vec{z}_i)) \in \dtest ~|~ \vec{z} \in X_S' \wedge g(\vec{z}) = f(\vec{z}) \}|}{|\dtest|}
\]

\cla{Similarly, given that if $g$ is robust on $\vec{x}$ and $N(\vec{x}) \subseteq X_S'$, then $g$ is also resilient on $\vec{x}$, we can introduce a lower bound $\hat{R}$ on resilience as:}
\[
\hat{R}(g,\dtest) = \dfrac{|\{(\vec{z}_i,f(\vec{z}_i)) \in \dtest ~|~ N(\vec{z})  \subseteq X_S' \wedge g \textnormal{ is robust on } \vec{z} \}|}{|\dtest|}
\]

\cla{Note that the pre-computation of $X_S'$ makes robustness verification straightforward for all the instances falling in $X_S'$, and that the computation of the resilience estimate $\hat{R}$ can potentially exploit any robustness verification method, which allows one to leverage existing work in the area. Moreover, we may use the accuracy of $\hat{r}$ with respect to the true robustness $r$ (established using an existing robustness verifier) as a proxy for the accuracy of $X_S'$, which allows us to assess the quality of the under-approximation $\hat{R}$ and advocate its adoption to mitigate the false sense of security provided by $r$. We leverage this idea in our experimental evaluation (Section~\ref{sec:experiments}).}

\revise{\subsection{Resilience vs. Global Robustness}}
\label{sec:res-vs-global}

\revise{Recent work proposed a technique to train \emph{globally robust} neural networks, which provide robustness guarantees for all the possible inputs, rather than just for the inputs in the test set~\cite{LeinoWF21}. The idea of globally robust neural networks, generalizable to arbitrary classifiers, is that the set of labels $\labels$ is extended with a special class $\bot$, used to denote that no reliable prediction is possible, because the instance is too close to the decision boundary of the classifier and thus potentially susceptible to evasion attacks. Global robustness requires that any two instances which are close enough to each other are either assigned the same prediction, or at least one of them is flagged as $\bot$. This property is intuitive and desirable, however, contrary to resilience, it cannot be used to verify the security of existing classifiers (which have not been trained to return the $\bot$ label).}

\revise{Rather, note that the proposed approach to resilience verification based on a data-independent stability analysis can be readily applied to transform any classifier into a globally robust classifier. In particular, given any classifier $g$ and a subset of the feature space $X_S' \subseteq \feats$ where $g$ is stable, one can define a globally robust classifier $g'$ as follows:
\[
g'(\vec{x}) =
\begin{cases}
g(\vec{x}) & \text{if } \vec{x} \in X_S' \\
\bot & \text{otherwise}
\end{cases}
\]}

\revise{Observe that the robustness estimate $\hat{r}$ previously defined provides a (local) robustness measure of the globally robust classifier $g'$ obtained through the previous construction.}

\section{Data-Independent Stability Analysis}
\label{sec:analysis}
We present here a data-independent stability analysis for decision trees and decision tree ensembles, which allows one to compute the two measures $\hat{r}$ and $\hat{R}$ defined in Section~\ref{sec:resilience-verification}, thus providing conservative estimates of robustness and (most importantly) resilience. The analysis is proved sound, i.e., we show that the portions of the feature space which are marked as stable by the analysis may only contain instances where the classifier is indeed stable. Proofs are given in Appendix~\ref{sec:proofs}.

\subsection{Preliminaries}
Our analysis leverages \emph{intervals} of real numbers. 
%and operations over them. 
Given $a,b \in \R \cup \{-\infty,+\infty\}$ with $a \leq b$, we use standard notation for intervals, using parentheses to represent open bounds and brackets to represent closed bounds. This leads to four possible types of intervals: $(a,b)$, $[a,b)$, $(a,b]$ and $[a,b]$. We use $I,J$ to range over intervals. 
A \emph{hyper-rectangle} $H \subseteq \feats$ is represented as a vector of intervals $\langle I_1,\ldots,I_d \rangle$ over $\R$. 

Given a decision tree $t$, we define its possible predictions over the hyper-rectangle $H$ as $t(H) = \{y \in \labels ~|~ \exists \vec{x} \in H: t(\vec{x}) = y\}$. Note that computing $t(H)$ is straightforward by means of a recursive tree traversal. Specifically, if $t = \lambda(y)$, then $t(H) = \{y\}$. If instead $t = \sigma(f,v,t_l,t_r)$, we define $t(H)$ with $H = \langle I_1,\ldots,I_d \rangle$ as follows:
\[
t(H) = 
\begin{cases}
t_l(H) & \text{if } I_f \cap (v,+\infty) = \emptyset \\
t_r(H) & \text{if } I_f \cap (-\infty,v] = \emptyset \\
t_l(H) \cup t_r(H) & \text{otherwise}
\end{cases}
\]

Finally, we extend predictions over hyper-rectangles from trees to tree ensembles, noted $T(H)$. The actual definition of $T(H)$ depends on the approach used by $T$ to aggregate the individual tree predictions to produce the ensemble prediction. For example, in the case of majority voting, we may let:
\[
T(H) =
\begin{cases}
\{y\} & \text{if } |\{t_i \in T ~|~ t_i(H) = \{y\}\}| > |T|/2 \\
\labels & \text{otherwise}
\end{cases}
\]

For soundness, we require $\{y \in \labels ~|~ \exists \vec{x} \in H: T(\vec{x}) = y\} \subseteq T(H)$, i.e., the set of the predictions $T(H)$ includes all the predictions that may be assigned to an instance in $H$. Proving that this requirement holds for the definition above is straightforward.

\revisee{In the following, given two intervals, we define their sum $I + J$ as the interval whose lower bound is the sum of the lower bounds and whose upper bound is the sum of the upper bounds; we require the bounds to be closed if and only if both the added bounds are closed. For example, $[1,3] + (4,6] = (5,9]$.
Moreover, given two hyper-rectangles, we define their sum $H + H'$ as the pointwise sum of their components (intervals). }

\subsection{Threat Model}
We assume each feature $f \in \{1,\ldots,d\}$ can be manipulated by paying a \emph{cost} $c_f \in \N$, which allows the attacker to add a perturbation $\delta \in \R$ arbitrarily drawn from an interval $\iatk_f$. We characterize the attacker's power in terms of a \emph{budget} $b$, which determines the maximum aggregate cost that the attacker can pay to manipulate features. We encode robust features which cannot be manipulated by setting $[0,0]$ as their perturbation interval.

\begin{definition}[Adversarial Manipulations]
Given an instance $\vec{x} \in \feats$, we define the set of the \emph{adversarial manipulations}, noted $A(\vec{x})$, as the set of the instances $\vec{z}$ such that there exists a set of features $F \subseteq \{1,\ldots,d\}$ such that:
\begin{enumerate}
\item For all $f \in F$, we have $z_f = x_f + \delta$ for some $\delta \in \iatk_f$.
\item For all $f \not\in F$, we have $z_f = x_f$.
\item We have $\sum_{f \in F} c_f \leq b$.
\end{enumerate}
\end{definition}

This threat model is inspired by traditional distance-based attackers from the adversarial ML literature and it is expressive enough to model a wide range of attacks, including those based on the $L_0$-norm and the $L_\infty$-norm which are traditionally used in prior work~\cite{abs-2004-03295,ChenZS0BH19}. In particular, note that:
\begin{itemize}
    \item $L_0$-norm attackers can be modeled by assuming that each feature can be perturbed in the interval $[-\infty,+\infty]$ by paying cost 1 and by setting the attacker's budget to the maximum $L_0$-distance assumed for the attack.
    \item $L_\infty$-norm attackers can be modeled by assuming that each feature can be perturbed in the interval $[-\delta,+\delta]$ by paying cost 1, where $\delta$ is the maximum $L_\infty$-distance assumed for the attack, and by setting the attacker's budget to the total number of features.
\end{itemize}

By using this threat model, we are able to design a relatively general analysis technique that readily applies to at least these two popular classes of attackers.

\subsection{Stability Analysis for Decision Trees}
Our analysis operates by annotating each node of the decision tree with a set of \emph{symbolic attacks}, which represent a set of instances along with their \emph{relevant} adversarial manipulations. A first insight of our analysis is that most adversarial manipulations are not relevant for the stability analysis of decision trees, because such classifiers operate by means of thresholds, hence only attacks which allow for traversing some threshold might lead to instability~\cite{CalzavaraLT19}. Formally, a symbolic attack $s$ has the shape $\symatk{\ipre_1, \ldots, \ipre_d}{\ipost_1, \ldots, \ipost_d}{k}$, where each $\ipre_i,\ipost_j$ is an interval on $\R$ and $k \in \N$. Intuitively, $s$ identifies the set of instances located within the hyper-rectangle $\langle \ipre_1, \ldots, \ipre_d \rangle$, called the \emph{pre-image} of the symbolic attack, along with their adversarial manipulations located within the hyper-rectangle $\langle \ipost_1, \ldots, \ipost_d \rangle$, called the \emph{post-image} of the symbolic attack; the \emph{cost} of such adversarial manipulations is bounded above by $k$. We make use of symbolic attacks to identify which nodes of the decision tree can be traversed by a set of instances as the result of adversarial manipulations against them; we require that when $k = 0$, the pre-image and the post-image coincide, i.e., the symbolic attack captures the case where no adversarial manipulation has taken place. We write $s.\pre$, $s.\post$ and $s.\paid$ to project the pre-image, the post-image and the cost of $s$ respectively.

Before presenting the formal details, we present the analysis at work on the decision tree of Figure~\ref{fig:tree} built on a feature space with two features. We assume an attacker who can manipulate at most one feature by adding a perturbation $\delta \in [-1,1]$. Formally, this is represented by having $\iatk_1 = \iatk_2 = [-1,1]$, $c_1 = c_2 = 1$ and $b = 1$. The result of the analysis in terms of node annotations are shown in Table~\ref{tab:analysis}. Node 1 is the root of the tree, hence the analysis cannot conclude anything about instances traversing the node, i.e., the symbolic attack in the node annotation models that all instances in the feature space traverse the root, no matter what the attacker does. Node 2, instead, is more interesting, because the analysis captures two cases via two symbolic attacks: an instance $\vec{x}$ might traverse the node either because $x_1 \leq 10$ and the attacker does nothing, or because $x_1 \in (10,11]$ is adversarially manipulated into the interval $(9,10]$. Note that, in the second case, the symbolic attack is assigned cost 1, which allows us to track that no further manipulation is possible because the budget has been exhausted. Even more interesting is the case of node 3, where we have three possibilities. In particular, an instance $\vec{x}$ might traverse the node in the following cases: $(i)$ $x_1 \leq 10$ and $x_2 \leq 5$, hence no adversarial manipulation is needed, $(ii)$ $x_1 \leq 10$ and $x_2 \in (5,6]$ is manipulated into $(4,5]$, or $(iii)$ $x_1 \in (10,11]$ is manipulated into $(9,10]$ and $x_2 \leq 5$. We do not have a case where both features are manipulated, because this would exceed the attacker's budget. A similar reasoning applies to the other nodes in the tree. Once the tree has been annotated, it is possible to check stability by inspecting the annotations in its leaves: we discuss this aspect of the analysis later in the section.

\begin{table}[t]
    \centering
    \caption{Analysis results for the decision tree of Figure~\ref{fig:tree}}
    \label{tab:analysis}
    \begin{tabular}{|l|l|}
    \hline
    1 & $\symatk{(-\infty,+\infty),(-\infty,+\infty)}{(-\infty,+\infty),(-\infty,+\infty)}{0}$ \\
    \hline
    2 & $\symatk{(-\infty,10],(-\infty,+\infty)}{(-\infty,10],(-\infty,+\infty)}{0}$ \\
      & $\symatk{(10,11],(-\infty,+\infty)}{(9,10],(-\infty,+\infty)}{1}$ \\
    \hline
    3 & $\symatk{(-\infty,10],(-\infty,5]}{(-\infty,10],(-\infty,5]}{0}$ \\
      & $\symatk{(-\infty,10],(5,6]}{(-\infty,10],(4,5]}{1}$ \\
      & $\symatk{(10,11],(-\infty,5]}{(9,10],(-\infty,5]}{1}$ \\
    \hline
    4 & $\symatk{(-\infty,10],(5,+\infty)}{(-\infty,10],(5,+\infty)}{0}$ \\
      & $\symatk{(-\infty,10],(4,5]}{(-\infty,10],(5,6]}{1}$ \\
      & $\symatk{(10,11],(5,+\infty)}{(9,10],(5,+\infty)}{1}$ \\
    \hline
    5 & $\symatk{(10,+\infty),(-\infty,+\infty)}{(10,+\infty),(-\infty,+\infty)}{0}$ \\
      & $\symatk{(9,10],(-\infty,+\infty)}{(10,11],(-\infty,+\infty)}{1}$ \\
    \hline
    6 & $\symatk{(10,+\infty),(-\infty,8]}{(10,+\infty),(-\infty,8]}{0}$ \\
      & $\symatk{(10,+\infty),(8,9]}{(10,+\infty),(7,8]}{1}$ \\
      & $\symatk{(9,10],(-\infty,8]}{(10,11],(-\infty,8]}{1}$ \\
    \hline
    7 & $\symatk{(10,+\infty),(8,+\infty)}{(10,+\infty),(8,+\infty)}{0}$ \\
      & $\symatk{(10,+\infty),(7,8]}{(10,+\infty),(8,9]}{1}$ \\
      & $\symatk{(9,10],(8,+\infty)}{(10,11],(8,+\infty)}{1}$ \\
    \hline
    \end{tabular}
\end{table}

Algorithm~\ref{alg:annotate} describes the annotation function for decision trees. We assume each node of the tree is enriched with an attribute $\sym$, used to store a set of symbolic attacks. The call $\Call{Annotate}{t,S}$ annotates the root of $t$ with the set of symbolic attacks $S$ passed as a parameter (line 2), then uses $S$ and the threshold information in the root to generate the annotations for the roots of the left and right sub-trees (lines 6-8); finally, the process goes down recursively (lines 9-10). When the annotation function is initially invoked on the root of the decision tree to analyze, we set $S = \{\symatk{(-\infty,+\infty)^d}{(-\infty,+\infty)^d}{0}\}$ as explained in the example.

\begin{algorithm}[t]
\caption{Decision tree annotation}
\label{alg:annotate}
\begin{algorithmic}[1]
\Function{Annotate}{$t,S$}
    \State $t.\sym \gets S$
    \If{$t = \sigma(f,v,t_l,t_r)$}
        \State $S_l \gets \emptyset$
        \State $S_r \gets \emptyset$
        \For{$s \in S$}
            \State $S_l \gets S_l \cup \Call{RefineLeft}{s,f,v}$
            \State $S_r \gets S_r \cup \Call{RefineRight}{s,f,v}$
        \EndFor
        \State $t_l \gets \Call{Annotate}{t_l,S_l}$
        \State $t_r \gets \Call{Annotate}{t_r,S_r}$
    \EndIf
    \State \Return{$t$}
\EndFunction
\end{algorithmic}
\end{algorithm}

\begin{algorithm}[t]
\caption{Refinement for the left and right sub-trees}
\label{alg:refine}
\begin{algorithmic}[1]
\Function{RefineLeft}{$s,f,v$}
    \State $S \gets \emptyset$
    \State $\langle \ipre_1,\ldots,\ipre_d \rangle \gets s.\pre$
    \State $\langle \ipost_1,\ldots,\ipost_d \rangle \gets s.\post$
    \State $k \gets s.\paid$
    \State $\langle \delta_l,\delta_r \rangle \gets \iatk_f$
    \If{$\ipost_f \cap (-\infty,v] \neq \emptyset$}
        \State $\jpost_f \gets \ipost_f \cap (-\infty,v]$
        \If{$\ipre_f = \ipost_f$}
            \State $\jpre_f \gets \ipre_f \cap (-\infty,v]$
        \Else 
            \State $\jpre_f \gets \ipre_f \cap (-\infty,v-\min(0,\delta_l)]$
        \EndIf
        \State $S \gets S \cup \{\symatk{\ipre_1,\ldots,\ipre_{f-1}, \jpre_f, \ipre_{f+1},\ldots \ipre_d}{\ipost_1,\ldots,\ipost_{f-1}, \jpost_f, \ipost_{f+1}, \ldots, \ipost_d}{k} \}$
    \EndIf
    \If{$\ipre_f = \ipost_f \wedge \delta_l < 0 \wedge \ipre_f \cap (v,v-\delta_l] \neq \emptyset \wedge k + c_f \leq b$}
        \State $\jpost_f \gets \ipost_f \cap (v+\delta_l,v]$
        \State $\jpre_f \gets \ipre_f \cap (v,v-\delta_l]$
        \State $S \gets S \cup \{\symatk{\ipre_1,\ldots,\ipre_{f-1}, \jpre_f, \ipre_{f+1},\ldots \ipre_d}{\ipost_1,\ldots,\ipost_{f-1}, \jpost_f, \ipost_{f+1}, \ldots, \ipost_d}{{k+c_f}} \}$
    \EndIf
    \State \Return{$S$}
\EndFunction

\State 

\Function{RefineRight}{$s,f,v$}
    \State $S \gets \emptyset$
    \State $\langle \ipre_1,\ldots,\ipre_d \rangle \gets s.\pre$
    \State $\langle \ipost_1,\ldots,\ipost_d \rangle \gets s.\post$
    \State $k \gets s.\paid$
    \State $\langle \delta_l,\delta_r \rangle \gets \iatk_f$
    \If{$\ipost_f \cap (v,+\infty) \neq \emptyset$}
        \State $\jpost_f \gets \ipost_f \cap (v,+\infty)$
        \If{$\ipre_f = \ipost_f$}
            \State $\jpre_f \gets \ipre_f \cap (v,+\infty)$
        \Else 
            \State $\jpre_f \gets \ipre_f \cap (v-\max(0,\delta_r),+\infty)$
        \EndIf
        \State $S \gets S \cup \{\symatk{\ipre_1,\ldots,\ipre_{f-1}, \jpre_f, \ipre_{f+1},\ldots \ipre_d}{\ipost_1,\ldots,\ipost_{f-1}, \jpost_f, \ipost_{f+1}, \ldots, \ipost_d}{k} \}$
    \EndIf
    \If{$\ipre_f = \ipost_f \wedge \delta_r > 0 \wedge \ipre_f \cap (v-\delta_r,v] \neq \emptyset \wedge k + c_f \leq b$}
        \State $\jpost_f \gets \ipost_f \cap (v,v+\delta_r]$
        \State $\jpre_f \gets \ipre_f \cap (v-\delta_r,v]$
        \State $S \gets S \cup \{\symatk{\ipre_1,\ldots,\ipre_{f-1}, \jpre_f, \ipre_{f+1},\ldots \ipre_d}{\ipost_1,\ldots,\ipost_{f-1}, \jpost_f, \ipost_{f+1}, \ldots, \ipost_d}{k+c_f} \}$
    \EndIf
    \State \Return{$S$}
\EndFunction
\end{algorithmic}
\end{algorithm}

\begin{algorithm*}[t]
\caption{Stability analysis for decision trees}
\label{alg:stability-tree}
\begin{algorithmic}[1]
\Function{Analyze}{$t$}
    \State $t \gets \Call{Annotate}{t, \{\symatk{(-\infty,+\infty)^d}{(-\infty,+\infty)^d}{0}\}}$
    \State $U \gets \emptyset$
    \For{$\lambda(y) \in t$}
        \For{$s \in \{\hat{s} \in \lambda(y).sym ~|~ \hat{s}.\paid = 0\}$}
            \For{$\lambda'(y') \in t$}
                \If{$y' \neq y$}
                    \For{$s' \in \{\hat{s} \in \lambda'(y').sym ~|~ \hat{s}.\paid > 0\}$}
                        \If{$s.\pre \cap s'.\pre \neq \emptyset$}
                            \State $s''.\pre \gets s.\pre \cap s'.\pre$
                            \State $s''.\post \gets s'.\post \cap (s''.\pre + \langle \iatk_1, \ldots, \iatk_d \rangle)$
                            \State $s''.\paid \gets s'.\paid$
                            \State $U \gets U \cup \{s''\}$
                        \EndIf
                    \EndFor
                \EndIf
            \EndFor
        \EndFor
    \EndFor
    \State \Return{$U$}
\EndFunction
\end{algorithmic}
\end{algorithm*}

The key part of the node annotation logic is implemented by the auxiliary functions {\sc RefineLeft} and {\sc RefineRight}, defined in Algorithm~\ref{alg:refine}. Given a symbolic attack $s$, a feature $f$ and the associated threshold $v$ from an internal node of the decision tree, the call $\Call{RefineLeft}{s,f,v}$ uses $s$ to generate a new set of symbolic attacks $S$ for the root of the left sub-tree (initially empty). Lines 7-13 account for the case where some instances in the post-image of $s$ already fall in the left sub-tree, i.e., the attacker does not need to spend budget to manipulate the feature $f$ so as to push some instances in the pre-image into the left sub-tree. In this case, $S$ is extended with a refined variant of $s$, where we track that the feature $f$ must belong to the interval $(-\infty,v]$ for the instances in the post-image (line 8). Also the pre-image of $s$ is refined in the left sub-tree: if $f$ was not attacked, we know that the feature $f$ must belong to the interval $(-\infty,v]$ for the instances in the pre-image as well (lines 9-10); otherwise, we still know that the attack could not push instances beyond the maximum negative perturbation $\delta_l < 0$, hence the feature $f$  must belong to the interval $(-\infty,v-\min(0,\delta_l)]$ for the instances in the pre-image (lines 11-12). Lines 14-17, instead, cover the case where some instances in the pre-image are close enough to the threshold $v$ to be pushed into the left sub-tree as the result of adversarial manipulations. In this case, provided that the attacker still has enough budget to spend, $S$ is extended with a refined variant of $s$, where we update both the post-image and the pre-image to reflect their proximity to the threshold $v$. More precisely, given the maximum negative perturbation $\delta_l < 0$, we track that the feature $f$ must belong to the interval $(v+\delta_l,v]$ for the instances in the post-image and to the interval $(v,v-\delta_l]$ for the instances in the pre-image, otherwise crossing the threshold would not be possible. The {\sc RefineRight} function performs an analogous reasoning for the right sub-tree, hence we omit a detailed explanation.

The stability analysis for decision trees is finally shown in Algorithm~\ref{alg:stability-tree}. The call $\Call{Analyze}{t}$ leverages the results of the tree annotation function to return a set of symbolic attacks $U$, identifying the portions of the feature space where the decision tree $t$ may be unstable. The function operates by looking for two leaves with different class predictions such that: $(i)$ the first leaf contains a symbolic attack $s$ of cost 0, i.e., no adversarial manipulation was performed on the pre-image of $s$, $(ii)$ the second leaf contains a symbolic attack $s'$ of cost greater than 0, i.e., the attacker manipulated the pre-image of $s'$, and $(iii)$ the pre-images of the two symbolic attacks $s,s'$ partially overlap, i.e., there exist some instances which might fall into a leaf with a different class than the original prediction due to adversarial manipulations (lines 4-9). In this case, the intersection of the pre-images identifies a portion of the feature space where the tree may be unstable and the post-image of $s'$ is refined to capture that the adversarial manipulations cannot push instances beyond the maximum allowed manipulation of the intersection of the two pre-images (lines 10-13). This is a conservative approximation, which accounts for all the possible adversarial manipulations.

To exemplify the output of the stability analysis, consider the node annotations in Table~\ref{tab:analysis}. The stability analysis returns the following symbolic attacks:
\begin{itemize}
    \item For leaf 3: $\symatk{(-\infty,10],(4,5]}{(-\infty,10],(5,6]}{1}$
    \item For leaf 4: $\symatk{(-\infty,10],(5,6]}{(-\infty,10],(4,5]}{1}$ and $\symatk{(9,10],(5,8]}{(10,11],(4,8]}{1}$
    \item For leaf 6: $\symatk{(10,+\infty),(7,8]}{(10,+\infty),(8,9]}{1}$ and $\symatk{(10,11],(5,8]}{(9,10],(5,9]}{1}$
    \item For leaf 7: $\symatk{(10,+\infty),(8,9]}{(10,+\infty),(7,8]}{1}$
\end{itemize}

The pre-images of these symbolic attacks identify the portions of the feature space where the decision tree is unstable. It is interesting to observe that leaves 4 and 6 contribute two portions of the feature space where the tree may be unstable, while leaves 3 and 7 only contribute one. The reason for this is that leaves 4 and 6 partially overlap on the values allowed for the second feature, i.e., the interval $(5,8]$. This means that it is possible to jump from leaf 4 to leaf 6 (and vice-versa) as the result of an attack targeting just the first feature, provided that the second feature falls into $(5,8]$. Conversely, leaves 3 and 7 have no overlap on any of the two features, hence an attack which manipulates just one feature cannot induce a jump between these two leaves. As a final comment, notice that the post-images are not needed at this stage of the analysis: we just collect them because they support the analysis of tree ensembles in the next section.

The soundness theorem for our analysis is given below. The theorem states that all the instances $\vec{x}$ where $t$ is unstable must fall in the pre-image of a symbolic attack contained in the set $U$ returned by the call $\Call{Analyze}{t}$. In other words, the union of the pre-images of the symbolic attacks in $U$ can only over-approximate the portion of the feature space where $t$ is unstable, i.e., $t$ must be stable on all the instances located outside such area. This allows us to compute the set $X_S'$ where $t$ is stable (cf. Section~\ref{sec:resilience-verification}) by subtracting the union of the pre-images of $U$ from the full feature space $\feats$.

\begin{theorem}[Soundness of Tree Analysis]
\label{thm:tree}
The call $\Call{Analyze}{t}$ returns a set of symbolic attacks $U$ such that, for every instance $\vec{x} \in \feats$ and every adversarial manipulation $\vec{z} \in A(\vec{x})$ such that $t(\vec{z}) \neq t(\vec{x})$, there exists $s \in U$ such that $\vec{x} \in s.\pre$ and $\vec{z} \in s.\post$.
\end{theorem}

\subsection{Stability Analysis for Tree Ensembles}
We now discuss how the stability analysis for decision trees can be generalized to tree ensembles by means of an iterative algorithm (Algorithm~\ref{alg:stability-ensemble}). The algorithm operates by refining a set of candidates $C$ where the ensemble $T$ may be unstable, initially set to the union of the symbolic attacks computed for the individual trees $t_i \in T$ (line 2-4). The algorithm also keeps track of an initially empty set of symbolic attacks $E$ where the analysis has ended because no further refinement of them is possible (line 5). 

Each iteration of the algorithm inspects all the candidates $s \in C$, distinguishing three cases (lines 6-15). If $T$ predicts the same label $y$ over both the pre-image and the post-image of $s$, then $T$ is stable on that portion of the feature space and $s$ is removed from $C$ (lines 8-9). Otherwise, the algorithm checks whether the predictions performed by $T$ over the pre-image and the post-image of $s$ share some common elements. If this is the case, then $T$ may be stable on a subset of the pre-image of $s$, yet this cannot be concluded at the current iteration; hence, $s$ is refined by splitting it into a set of smaller symbolic attacks (lines 11-12). Any splitting criterion would work, as long as it satisfies the soundness condition defined in the theorem below. If instead the predictions performed by $T$ over the pre-image and the post-image of $s$ are disjoint, there is no way of proving that $T$ is stable even on a subset of the pre-image of $s$, hence $s$ is moved from $C$ to $E$ to avoid further refinements (lines 14-15). 

The algorithm may implement an arbitrary stopping condition, e.g., $C$ is empty or a maximum number of iterations has been performed, as we do in our implementation. Similarly to the stability analysis for decision trees, the algorithm eventually returns a set of symbolic attacks $C \cup E$, whose union of the pre-images over-approximates the portion of the feature space where $T$ is unstable. The soundness of the analysis is formalized by the following theorem, which is the natural generalization to ensembles of Theorem~\ref{thm:tree}.

\begin{algorithm}[t]
\caption{Stability analysis for tree ensembles}
\label{alg:stability-ensemble}
\begin{algorithmic}[1]
\Function{Analyze}{$T$}
    \State $C \gets \emptyset$
    \For{$t_i \in T$}
        \State $C \gets C \cup \Call{Analyze}{t_i}$
    \EndFor
    \State $E \gets \emptyset$
    \While{stopping condition is not met}
        \For{$s \in C$}
            \If{$\exists y: T(s.\pre) = T(s.\post) = \{y\}$}
                \State $C \gets C \setminus \{s\}$
            \Else
                \If{$T(s.\pre) \cap T(s.\post) \neq \emptyset$}
                    \State $C \gets (C \setminus \{s\}) \cup \Call{Split}{s}$
                \Else
                    \State $C \gets C \setminus \{s\}$
                    \State $E \gets E \cup \{s\}$
                \EndIf
            \EndIf
        \EndFor
    \EndWhile
    \State \Return{$C \cup E$}
\EndFunction
\end{algorithmic}
\end{algorithm}

\begin{theorem}[Soundness of Tree Ensemble Analysis]
\label{thm:ensemble}
Assume the following hypotheses:
\begin{itemize}
    \item $T(H)$ is sound, i.e., it satisfies the following:
    \[
    \{y \in \labels ~|~ \exists \vec{x} \in H: T(\vec{x}) = y\} \subseteq T(H)
    \]

    \item The splitting procedure is sound, i.e., for all symbolic attacks $s$, if there exist $\vec{x} \in \feats$ and $\vec{z} \in A(\vec{x})$ such that $\vec{x} \in s.\pre$ and $\vec{z} \in s.\post$, then there exists $s' \in \Call{Split}{s}$ such that $\vec{x} \in s'.\pre$ and $\vec{z} \in s'.\post$.
\end{itemize}

The call $\Call{Analyze}{T}$ returns a set of symbolic attacks $U$ such that, for every instance $\vec{x} \in \feats$ and every adversarial manipulation $\vec{z} \in A(\vec{x})$ such that $T(\vec{z}) \neq T(\vec{x})$, there exists $s \in U$ such that $\vec{x} \in s.\pre$ and $\vec{z} \in s.\post$.
\end{theorem}

\subsection{Implementation}
We implemented our stability analysis for decision trees and decision tree ensembles (based on majority voting). The output of the stability analysis is used to compute lower bounds of robustness $\hat{r}$ and resilience $\hat{R}$ for a given model and test set, as discussed in Section~\ref{sec:resilience-verification}.

We discuss below selected aspects of the implementation, which is a rather direct translation of our pseudo-code. A first point to note is that the set of candidates $C$ is implemented by means of a min priority queue and only the top $k$ symbolic attacks can be split at each loop iteration ($k = 0.05 \cdot |C|$ by default) to mitigate the growth of $C$. The priority queue is ordered according to the following criterion: each symbolic attack $s$ is first assigned a pair $(n_c,n_u)$, where $n_c$ is a counter used to keep track of how many splits have been performed to produce $s$ and $n_u$ is the number of ``undecided'' trees $t_i$ such that $|t_i(s.\pre)| > 1$; pairs are then ordered according to the standard lexicographic order. In this heuristics, $n_c$ acts as a penalization factor to ensure that the algorithm does not split the same symbolic attacks too many times, but rather tries to process all the symbolic attacks at least once even when the number of iterations is relatively small; symbolic attacks with the same value of $n_c$ are split by prioritizing symbolic attacks with a small number of undecided trees $n_u$, because they are intuitively easier to certify and should be analyzed earlier.

A second relevant aspect to discuss is the implementation of the splitting function. Given the symbolic attack $s$, our implementation of $\Call{Split}{s}$ operates as follows: it first identifies a feature $f$ and a threshold $v$ such that $v$ falls in the $f$-th component of $s.\pre$; then, if $\iatk_f = [\delta_l,\delta_r]$, it splits $s.\pre$ in (at most four) hyper-rectangles based on the thresholds $v + \delta_l$, $v$, $v + \delta_r$. For example given the interval $(a, b]$, if $v$, $v + \delta_l$ and $v + \delta_r$ are inside $(a, b]$, the $\Call{Split}{s}$ function divides the interval in the following four intervals: $(a, v + \delta_l]$, $(v + \delta_l, v]$, $(v, v+ \delta_r]$ and $(v + \delta_r, b]$. Finally, it uses these intervals as the pre-images of the new symbolic attacks, computing the corresponding post-images by intersecting $s.\post$ with the maximum perturbation applicable to the pre-images. It is easy to show that this implementation enjoys the required soundness condition for the splitting procedure.

Finally, we note that Algorithm~\ref{alg:stability-ensemble} can be parallelized by partitioning $C$ across different threads and by joining the analysis results at the end. In particular, we first build the priority queue $C$ and then we distribute it across threads using a round robin algorithm, which is useful to ensure that no thread is penalized by the prevalence of symbolic attacks which are expected to be hard to analyze. This scheme ensures a deeper exploration of $C$ when the analysis terminates before convergence after a fixed number of iterations. Our implementation supports a configurable number of threads.

\section{Experimental Evaluation}
\label{sec:experiments}
We finally report on the experimental evaluation of our resilience verification technique. We first discuss the setup and the research questions, then we present the key experiments and results.

\subsection{Experimental Setup}

We evaluate our proposal using \empirical{three} publicly available datasets from LIBSVM Data\footnote{https://www.csie.ntu.edu.tw/~cjlin/libsvmtools/datasets/}: Diabetes, Cod-RNA and Breast Cancer.
%\footnote{\url{www.csie.ntu.edu.tw/~cjlin/libsvmtools/datasets/binary.html#diabetes}}\footnote{\url{www.csie.ntu.edu.tw/~cjlin/libsvmtools/datasets/binary.html#cod-rna}}\footnote{\url{www.csie.ntu.edu.tw/~cjlin/libsvmtools/datasets/binary/diabetes}}.
Datasets are divided into a training test $\dtrain$ and a test set $\dtest$ by using a 80-20 splitting with stratified random sampling. We normalize each feature in the interval $[0,1]$ and we train different types of classifiers on $\dtrain$, i.e., decision trees of different depths and random forests including different numbers of trees (of depth \empirical{3}). We cover both standard models trained using the popular sklearn library\footnote{\url{https://scikit-learn.org/stable/}} and robust models trained using a state-of-the art adversarial ML algorithm for decision trees, called TREANT~\cite{CalzavaraLTAO20}.

For each classifier, we leverage the test set $\dtest$ to compute different measures: accuracy $a$, robustness $r$, its under-approximation $\hat{r}$, and the under-approximation of resilience $\hat{R}$. The robustness $r$ is computed using an exact verification technique from related work~\cite{abs-2004-03295}, while $\hat{r}$ and $\hat{R}$ are under-approximations computed by our analysis (cf. Section~\ref{sec:resilience-verification}). Similarly to the existing literature, for simplicity we consider a synthetic threat model where each feature can be adversarially corrupted by adding a perturbation $\delta$, which is defined in Table~\ref{tab:perturbs} for the different datasets. \revise{The value of $\delta$ depends on the dataset, because different datasets are drawn from different distributions, hence attacks that are effective on models trained over a given dataset may be too strong or too weak for models trained on a different dataset~\cite{Chen0QLJW21}. The number of features that can be corrupted is defined by the attacker's budget $b$ set in the experiments.}

\begin{table}[]
    \centering
    \caption{Dataset Perturbations}
    \label{tab:perturbs}
    \begin{tabular}{c|c}
    \textbf{Dataset} & \textbf{Perturbation $\delta$} \\
    \hline
    diabetes & 0.03 \\
    cod-rna & 0.07 \\
    breast-cancer & 0.15 
    \end{tabular}
\end{table}

When estimating resilience, we assume the neighborhood $N(\vec{x}) = \{\vec{z} \in \feats ~|~ ||\vec{z} - \vec{x}||_\infty \leq \varepsilon\}$ for a given value of $\varepsilon$ defined in the experiments. Observe that the actual resilience $R$ is unknown and only the estimate $\hat{R}$ can be computed by our analysis. \revise{Our experiments are designed to investigate the following research questions:}
\begin{enumerate}
    \item May the theoretical shortcomings of robustness actually occur in practice?
    
    \item Can we compute accurate resilience estimates by means of our data-independent stability analysis and are these estimates practically useful?
    
    \item What is the impact of the parameter $\varepsilon$ on the resilience estimates that we can compute?
    
    \item What is the performance of our resilience verification technique in terms of running times and how is it affected by the attacker's budget $b$?
\end{enumerate}

\revise{All the experiments are performed using 20 threads.}

\subsection{Shortcomings of Robustness}
\revise{The first experiment we carry out motivates our study by showing that the relevant shortcomings of robustness identified on paper might also occur in practical scenarios. To do that, we use the original test set $\dtest$ of the different datasets to craft 100 synthetic test sets $\dtest^1,\ldots,\dtest^{100}$ obtained by replacing each instance $\vec{x}$ in $\dtest$ with a randomly sampled instance $\vec{z} \in N(\vec{x})$. We then compute the robustness of the trained classifiers over all the test sets $\dtest^i$, reporting the best and worst obtained results to understand to which extent a ``lucky'' sample of the data distribution might give a false sense of security. To ensure that the synthetic test sets are still representative of the same data distribution used for training, we only consider cases where the classifier roughly preserves the same accuracy computed on the original test set $\dtest$.}

\revise{Table~\ref{tab:motivation} presents the experimental results of our evaluation, assuming an attacker with budget $b = 1$. The table reports for the different datasets the worst robustness $r_{min}$ and the best robustness $r_{max}$ computed over all the generated test sets for different values of $\varepsilon$, as well as the corresponding values of accuracy, noted $a_{min}$ and $a_{max}$ respectively. The table also includes the accuracy $a$ and the robustness $r$ computed on the original test set $\dtest$. In the table, we mark in bold the cases where the gap between $r_{min}$ and $r_{max}$ is at least 0.04. The experiments are performed on ensembles with \empirical{7} decision trees.}

\begin{table*}[t]
    \caption{Shortcomings of robustness (for fixed $b = 1$)}
    \label{tab:motivation}
    \centering
    \begin{tabular}{c|c||c|c|c|c|c|c||c|c|c|c|c|c}
    \multicolumn{2}{c}{} & \multicolumn{6}{c}{Standard Models} & \multicolumn{6}{c}{Robust Models} \\
    Dataset & $\varepsilon$ & $a$ & $a_{min}$ & $a_{max}$ & $r$ & $r_{min}$ & $r_{max}$ & $a$ & $a_{min}$ & $a_{max}$ & $r$ & $r_{min}$ & $r_{max}$ \\
    \hline
    \multirow{4}{*}{diabetes}
    & 0.01 &0.714 &0.708 &0.721 &0.649 &0.643 &0.662 &0.727 &0.721 &0.727 &0.714 &0.675 &0.714 \\
    & 0.02 &0.714 &0.708 &0.714 &0.649 &0.630 &0.662 &0.727 &0.714 &0.740 &0.714 &{\bf 0.669} &{\bf 0.721} \\
    & 0.03 &0.714 &0.688 &0.714 &0.649 &{\bf 0.636} &{\bf 0.682} &0.727 &0.721 &0.747 &0.714 &{\bf 0.669} &{\bf 0.727} \\
    & 0.04 &0.714 &0.688 &0.727 &0.649 &{\bf 0.630} &{\bf 0.701} &0.727 &0.708 &0.747 &0.714 &{\bf 0.675} &{\bf 0.734} \\
    \hline
    \multirow{4}{*}{cod-rna}
    & 0.01 & 0.775& 0.774& 0.775& 0.686& 0.676& 0.690& 0.750& 0.748& 0.753& 0.714& 0.710& 0.721 \\
    & 0.02 & 0.775& 0.773& 0.775& 0.686& 0.665& 0.686& 0.750& 0.749& 0.758& 0.714& 0.711& 0.725\\
    & 0.03 & 0.775& 0.773& 0.775& 0.686& 0.657& 0.686& 0.750& 0.750& 0.760& 0.714& 0.705& 0.723 \\
    & 0.04 & 0.775& 0.768& 0.775& 0.686& 0.650& 0.686& 0.750& 0.752& 0.761& 0.714& 0.703& 0.723 \\
    \hline
    \multirow{4}{*}{breast-cancer}
    & 0.05& 0.948& 0.948& 0.948& 0.926& 0.926& 0.941& 0.970& 0.933& 0.963& 0.956& {\bf 0.919}& {\bf 0.963} \\
    & 0.06& 0.948& 0.933& 0.956& 0.926& {\bf 0.911}& {\bf 0.956}& 0.970& 0.933& 0.970& 0.956& {\bf 0.911}& {\bf 0.963} \\
    & 0.07& 0.948& 0.941& 0.956& 0.926& {\bf 0.904}& {\bf 0.956}& 0.970& 0.926& 0.963& 0.956& {\bf 0.911}& {\bf 0.963} \\
    & 0.08& 0.948& 0.933& 0.970& 0.926& {\bf 0.904}& {\bf 0.963}& 0.970& 0.926& 0.956& 0.956& {\bf 0.904}& {\bf 0.956} \\
    \hline
    \end{tabular}
\end{table*}

\revise{The results show that the size of the interval $[r_{min},r_{max}]$ is significant in most cases, typically ranging from 0.04 to 0.07 for the highest values of $\varepsilon$, while the size of the interval $[a_{min},a_{max}]$ is relatively small in comparison, ranging approximately from 0.01 to 0.04. For example, the robustness of the standard model trained over the breast-cancer dataset suffers from a fluctuation of around 0.05 for $\varepsilon = 0.07$, while the corresponding accuracy fluctuates of just 0.01. Our experiments show that robustness is generally more sensitive to small amounts of noise than accuracy. Remarkably, this observation applies to both standard and robust models trained using TREANT. Robust models offer an improved robustness over standard models, however the interval $[r_{min},r_{max}]$ may have a significant size also for them, i.e., roughly 0.06 in the worst case. This shows that a security evaluation based on robustness may give a false sense of security for both standard and robust models. We also remark that our experiment still provides a conservative account of the actual limitations of robustness, being based on just 100 synthetic test sets: the actual gap between $r_{min}$ and $r_{max}$ within the neighborhood $N$ may be larger in practice.}

\subsection{Effectiveness of Resilience Verification}
We now investigate the effectiveness of our resilience verification technique. To do that, we would like to show that our estimate $\hat{R}$ is an accurate under-approximation of the actual resilience $R$ and that resilience significantly mitigates the shortcomings of robustness. Unfortunately, since the actual value of resilience is unknown, we can only provide a best-effort answer to the first point. Our evaluation is based on two independent experiments:
\begin{enumerate}
\item In the first one, we operate by comparing the similarity between the actual robustness $r$ and its estimate $\hat{r}$ computed by our analysis. We consider the similarity between $r$ and $\hat{r}$ as a proxy for the precision of the stability analysis underlying our resilience verification technique: the more $r$ and $\hat{r}$ are close to each other, the more the stability analysis is effective at detecting the portions of the feature space where the classifier is stable, which suggests that also the estimate $\hat{R}$ is precise, being based on the same stability analysis.

\item \revise{In the second one, we refer to the experiment in the previous section and we observe that, if a classifier is not robust on the instance $\vec{z}$ belonging to some $\dtest^i$ with $i \in [1,100]$, then there must exist $\vec{x}$ in $\dtest$ such that the classifier is not robust on at least one instance in $N(\vec{x})$ by construction. This allows one to construct an additional test set $\overline{\dtest}$, corresponding to the ``most unlucky'' sampling in the neighborhood of the original $\dtest$, i.e., the one with the lowest robustness $\overline{r}$. The measure $\overline{r}$ is interesting, because it is based on an exact robustness verification technique: if $\overline{r}$ is close to $\hat{R}$, then we have a proof that most instances where the classifier is not considered resilient by our analysis are indeed insecure with respect to some evasion attacks.}
\end{enumerate}

Note that the second experiment does not just prove the precision of our approximated resilience verification technique, but it also gives a clear security interpretation of the benefits of resilience over robustness.

\begin{table*}[t]
    \caption{Computed measures for different datasets and models (for fixed $b = 1$)}
    \label{tab:results}
    \centering
    \begin{tabular}{c|c|c|c||c|c|c|c|c||c|c|c|c|c}
    \multicolumn{3}{c}{} & \multicolumn{5}{c}{Standard Models} & \multicolumn{5}{c}{Robust Models} \\
    Dataset & $\varepsilon$ & \# Trees & Depth & $a$ & $r$ & $\hat{r}$ & $\overline{r}$ & $\hat{R}$ & $a$ & $r$ & $\hat{r}$ & $\overline{r}$ & $\hat{R}$ \\
    \hline
    \multirow{6}{*}{diabetes} & \multirow{6}{*}{0.01}
      & 1 & 3 & 0.675 & 0.623 & 0.623 & 0.623 & 0.623 & 0.682 & 0.643 & 0.643 & 0.643 & 0.643 \\
    & & 1 & 5 & 0.721 & 0.636 & 0.636 & 0.617 & 0.617 & 0.675 & 0.636 & 0.636 & 0.630 & 0.630 \\
    & & 1 & 7 & 0.727 & {\bf 0.610} & 0.610 & {\bf 0.539} & {\bf 0.539} & 0.695 & 0.675 & 0.675 & 0.636 & 0.636 \\
    \cline{3-14}
    & & 5 & 3 & 0.708 & 0.662 & 0.643 & 0.656 & 0.636 & 0.727 & {\bf 0.714} & 0.701 & {\bf 0.675} & {\bf 0.662} \\
    & & 7 & 3 & 0.714 & 0.649 & 0.630 & 0.636 & 0.623 & 0.727 & {\bf 0.714} & 0.708 & {\bf 0.675} & {\bf 0.662} \\
    & & 9 & 3 & 0.747 & 0.656 & 0.630 & 0.623 & 0.617 & 0.753 & {\bf 0.740} & 0.727 & {\bf 0.695} & {\bf 0.688} \\
    \hline
    \multirow{6}{*}{cod-rna} & \multirow{6}{*}{0.01}
      & 1 & 3 & 0.774 & 0.683 & 0.683 & 0.638 & 0.637 & 0.750 & 0.714 & 0.714 & 0.698 & 0.698 \\
    & & 1 & 5 & 0.874 & {\bf 0.433} & 0.433 & {\bf 0.334} & {\bf 0.330} & 0.810 & {\bf 0.703} & 0.703 & {\bf 0.643} & {\bf 0.641} \\
    & & 1 & 7 & 0.804 & {\bf 0.575} & 0.575 & {\bf 0.474} & {\bf 0.472} & 0.810 & {\bf 0.700} & 0.700 & {\bf 0.637} & {\bf 0.634} \\
    \cline{3-14}
    & & 5 & 3 & 0.775 & {\bf 0.686} & 0.672 & {\bf 0.639} & {\bf 0.621} & 0.752 & 0.715 & 0.707 & 0.698 & 0.691 \\
    & & 7 & 3 & 0.775 & {\bf 0.686} & 0.666 & {\bf 0.640} & {\bf 0.612} & 0.750 & 0.714 & 0.713 & 0.698 & 0.697 \\
    & & 9 & 3 & 0.769 & {\bf 0.677} & 0.663 & {\bf 0.625} & {\bf 0.605} & 0.750 & 0.714 & 0.713 & 0.698 & 0.697 \\
    \hline
    \multirow{6}{*}{breast-cancer} & \multirow{6}{*}{0.05}
      & 1 & 3 & 0.948 & 0.904 & 0.904 & 0.874 & 0.874 & 0.948 & {\bf 0.941} & 0.941 & {\bf 0.859} & {\bf 0.859} \\
    & & 1 & 5 & 0.956 & 0.911 & 0.911 & 0.874 & 0.874 & 0.956 & {\bf 0.941} & 0.941 & {\bf 0.793} & {\bf 0.793} \\
    & & 1 & 7 & 0.948 & 0.904 & 0.904 & 0.867 & 0.867 & 0.956 & {\bf 0.941} & 0.941 & {\bf 0.793} & {\bf 0.793} \\
    \cline{3-14}
    & & 5 & 3 & 0.941 & {\bf 0.926} & 0.904 & {\bf 0.904} & {\bf 0.867} & 0.978 & {\bf 0.948} & 0.941 & {\bf 0.889} & {\bf 0.881} \\
    & & 7 & 3 & 0.948 & 0.926 & 0.911 & 0.926 & 0.881 & 0.970 & {\bf 0.956} & 0.941 & {\bf 0.889} & {\bf 0.874} \\
    & & 9 & 3 & 0.963 & {\bf 0.941} & 0.919 & {\bf 0.933} & {\bf 0.889} & 0.970 & {\bf 0.956} & 0.933 & {\bf 0.889} & {\bf 0.867} \\
    \hline
    \end{tabular}
\end{table*}

Table~\ref{tab:results} shows the experimental results of our evaluation, assuming $b = \empirical{1}$. The results highlight that the estimate $\hat{r}$ is a rather precise under-approximation of the actual robustness $r$: in particular, for individual decision trees $\hat{r}$ always coincides with $r$. As to the tree ensembles, the gap between the two measures increases, yet it is still quite small (roughly 0.02) for standard models and most often negligible for robust models trained using TREANT, hence we expect that also $\hat{R}$ is a reasonably accurate estimate of the actual resilience $R$. The table also shows that the gap between $r$ and $\hat{R}$ may be quite significant, both for standard and robust models, often reaching a value of around \empirical{0.06} and even reaching \empirical{0.10} in some cases. \revise{We highlight in bold the cases where the gap between $r$ and $\hat{R}$ is at least 0.05. Remarkably, it is apparent that the resilience estimate $\hat{R}$ provides a much more realistic security assessment than robustness, because the value of $\overline{r}$ is much closer to $\hat{R}$ than to $r$ in the very large majority of cases. Since $\overline{r}$ captures effective evasion attacks against instances within close neighborhoods of the test set, this confirms that $\hat{R}$ is not overly conservative in practice.}

%\subsection{Role of the Parameters $\varepsilon$ and $b$}
\revise{We finally assess the role of the parameter $\varepsilon$ on our resilience verification technique. For space reasons, we only focus on models trained on the Diabetes dataset using the TREANT algorithm. In particular, we set $b = 1$ and we compute different resilience estimates for different values of $\varepsilon$. Of course, we expect resilience to decrease when increasing the value of $\varepsilon$, because the stability guarantees required on the classifier become more demanding. Still, it is interesting to understand whether the quality of the resilience estimate $\hat{R}$ is affected by the value of $\varepsilon$: to understand this, we compare $\hat{R}$ against $\overline{r}$, because we would like the two measures to be relatively close to each other. Figure~\ref{fig:var-eps} plots how the value of our resilience estimate $\hat{R}$ and $\overline{r}$ decrease when increasing $\varepsilon$ from \empirical{0.01} to \empirical{0.05}. The figure shows that $\hat{R}$ and $\overline{r}$ are consistently close to each other, with a maximum difference of \empirical{0.02}. This shows that the computed resilience estimates always capture possible evasion attacks, i.e., the precision of our approximated analysis does not downgrade when increasing the value of $\varepsilon$.}

\begin{figure}[t]
    \centering
    \includegraphics[scale=0.45]{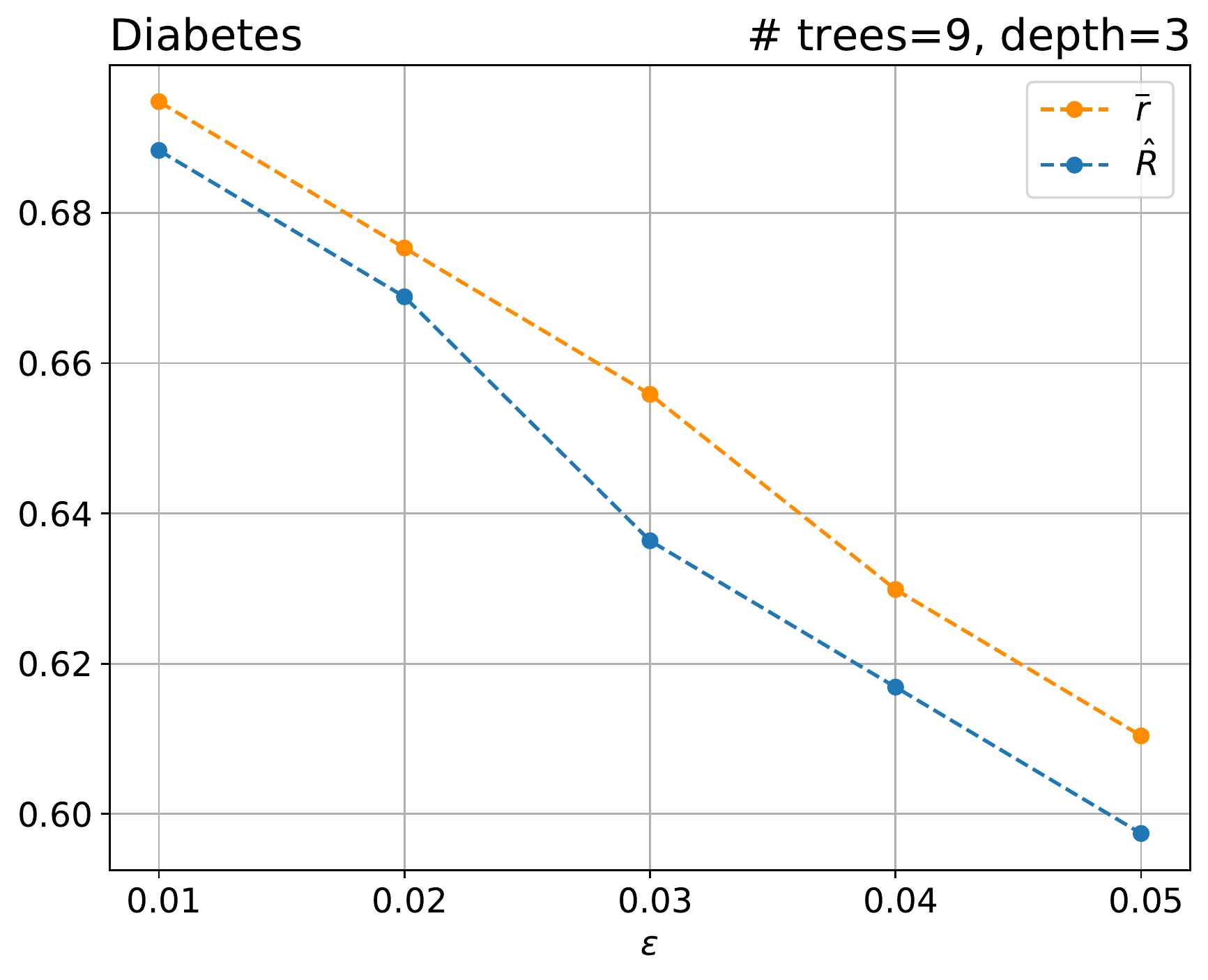}
    \caption{Resilience estimates for different values of $\varepsilon$}
    \label{fig:var-eps}
\end{figure}

% \revise{In the second experiment, we set $\varepsilon = \empirical{0.01}$ and we compute different resilience estimates for different values of $b$. Again, we expect resilience to decrease when increasing the value of $b$, because robustness drops and resilience is bounded above by robustness, however we may assess how much our resilience estimate $\hat{R}$ is influenced by $b$ through a comparison with the trend of $\overline{r}$. Figure~\ref{fig:var-budget} plots how the value of our resilience estimate $\hat{R}$ and $\overline{r}$ decrease when increasing $b$ from \empirical{1} to \empirical{4}, showing that $\hat{R}$ and $\overline{r}$ are still consistently close to each other, with a maximum difference of \empirical{0.02}.}

% \begin{figure}
%     \centering
%     \includegraphics[scale=0.45]{fig/Resilience_x_Budget.pdf}
%     \caption{Resilience estimate for different values of $b$}
%     \label{fig:var-budget}
% \end{figure}

\subsection{Performance Evaluation}
We finally investigate the performance of our resilience verification technique. For space reasons, we only focus on models trained on the Diabetes dataset using the TREANT algorithm. The decision tree analysis is based on a simple tree traversal, hence expected to be very efficient. Our experimental evaluation confirms this intuition: when varying the tree depth from \empirical{3} to \empirical{15}, the analysis always terminates in less than one second. 

The performance of the ensemble analysis is subtler to assess though, because that analysis is based on an iterative algorithm and its performance crucially depends on the number of iterations. We are interested in two aspects here:
\begin{enumerate}
    \item Understanding how much the analysis time (up to convergence) changes when increasing the ensemble size.
    
    \item Understanding how much the quality of the robustness and resilience estimates $\hat{r},\hat{R}$ changes when increasing the number of iterations, while keeping the same ensemble size.
\end{enumerate}

The first point provides insights on the scalability of the analysis to increasingly larger models, while the second point allows one to understand whether it is possible to compute useful robustness and resilience estimates even when the analysis becomes intractable and the number of iterations is limited to forcefully stop the analysis before convergence. Indeed, our analysis was deliberately designed to support iterative refinements and parallelization to be equipped against the exponential complexity blowup underlying the verification of decision tree ensembles~\cite{TornblomN20}.

Figure~\ref{fig:perf-ensembles} shows how the analysis time changes when increasing the size of the ensemble from 9 to 17. Small ensembles with 9 trees can be analyzed in a matter of seconds, while larger ensembles with 17 trees can be analyzed in around \empirical{16} minutes. We consider this result promising, because the stability analysis is data-independent, i.e., it can be computed just once and then applied to establish different properties on different test sets. We expect the analysis times to be further improvable by sacrificing a bit of precision, e.g., by aggregating together symbolic attacks which are close to each other. Our next experiment also provides positive results with respect to the scalability of the analysis to larger ensembles.

\begin{figure}[t]
    \centering
    \includegraphics[scale=0.4]{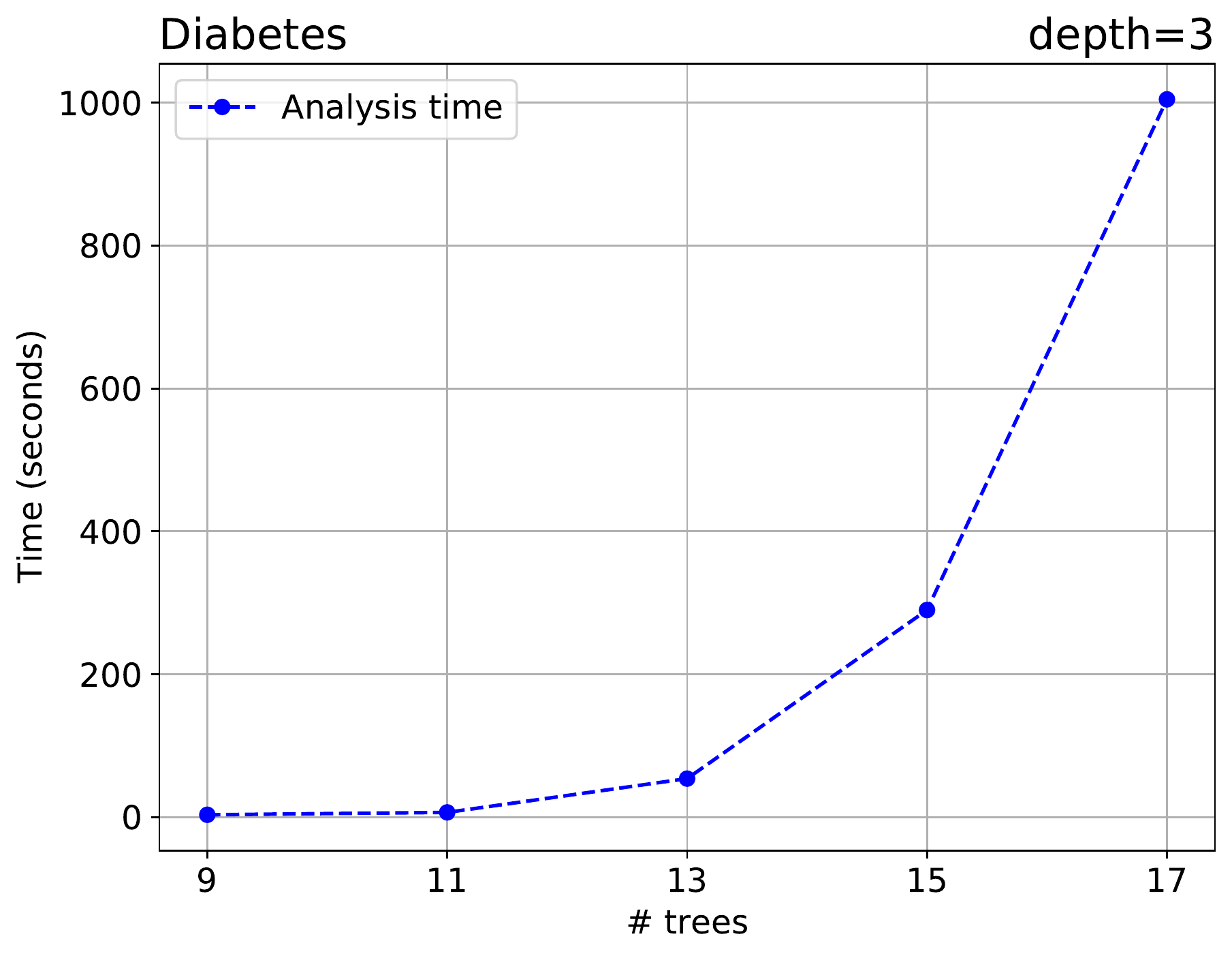}
    \caption{Analysis times when varying the size of the ensemble}
    \label{fig:perf-ensembles}
\end{figure}

Figure~\ref{fig:perf-ensembles-iter} shows how the value of the estimates $\hat{r}$ and $\hat{R}$ computed on an ensemble of \empirical{17} trees change when increasing the number of analysis iterations. The figure shows a desirable trend, with a significant increase of the estimates of robustness and resilience in the first 120 iterations before reaching a plateau. This means that it is possible to establish reasonably accurate estimates of robustness and resilience even with a limited number of iterations of the analysis, which is again important to support scalability, because useful results can be established also before analysis convergence. Indeed, our analysis was designed to prioritize portions of the feature space which are intuitively easier to prove as stable (or not).

\begin{figure}[t]
     \centering
     \includegraphics[scale=0.4]{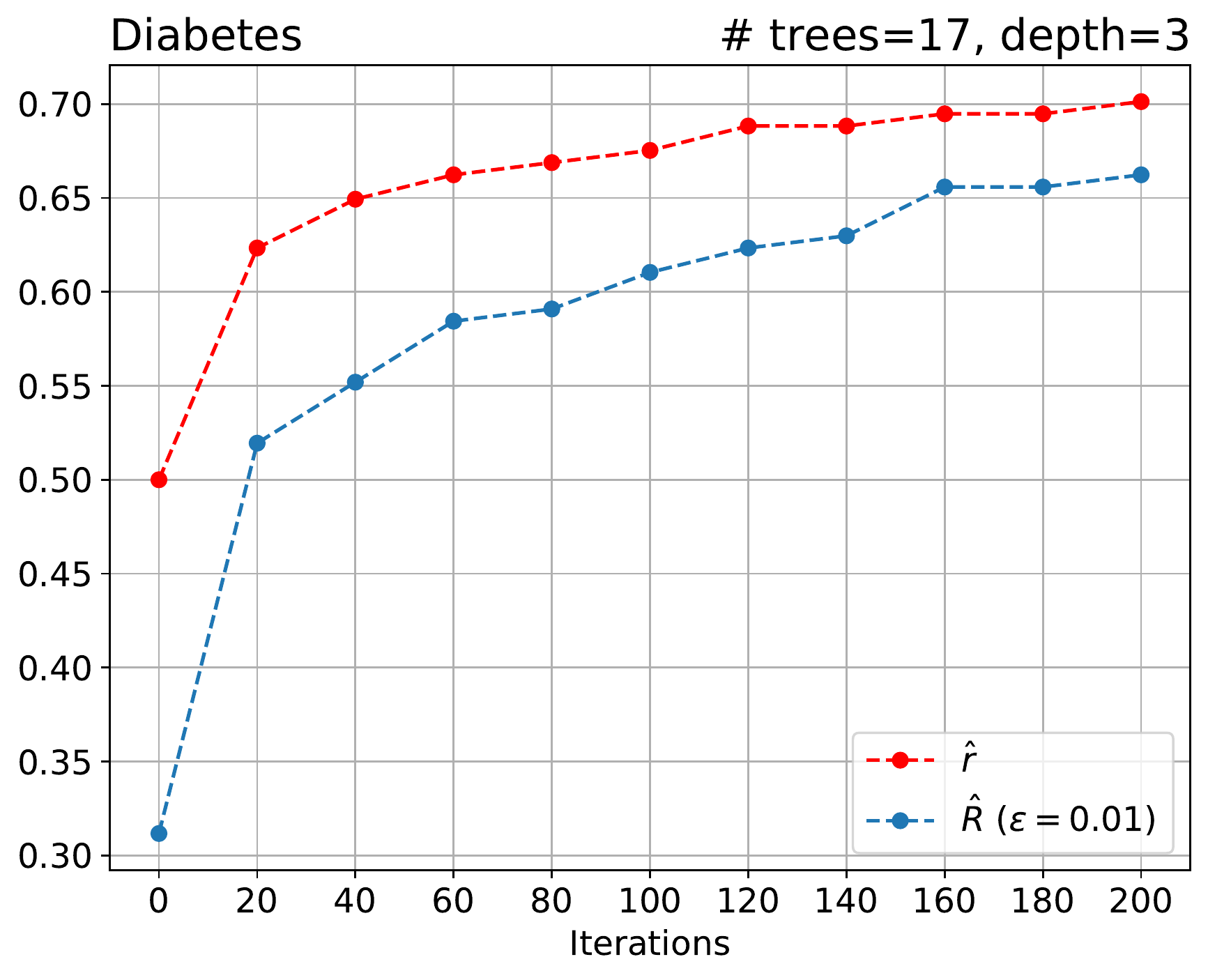}
     \caption{Analysis times when varying the number of analysis iterations}
     \label{fig:perf-ensembles-iter}
\end{figure}

\revise{The last experiment we carry out assesses how the attacker's budget $b$ may impact on the performance of resilience verification. In particular, we investigate how the analysis times change for different values of $b$, using an ensemble of \empirical{11} trees (trained with $b = 5$) analyzed up to convergence. The results of the experiment are shown in Figure~\ref{fig:perf-ensembles-budget}. As we can see, the impact of the attacker's budget on the analysis times is much more limited than the impact of the size of the ensemble. The analysis times just range from around \empirical{12} seconds to around \empirical{40} seconds when varying the attacker's budget from 1 to 5.}

\begin{figure}
    \centering
    \includegraphics[scale=0.4]{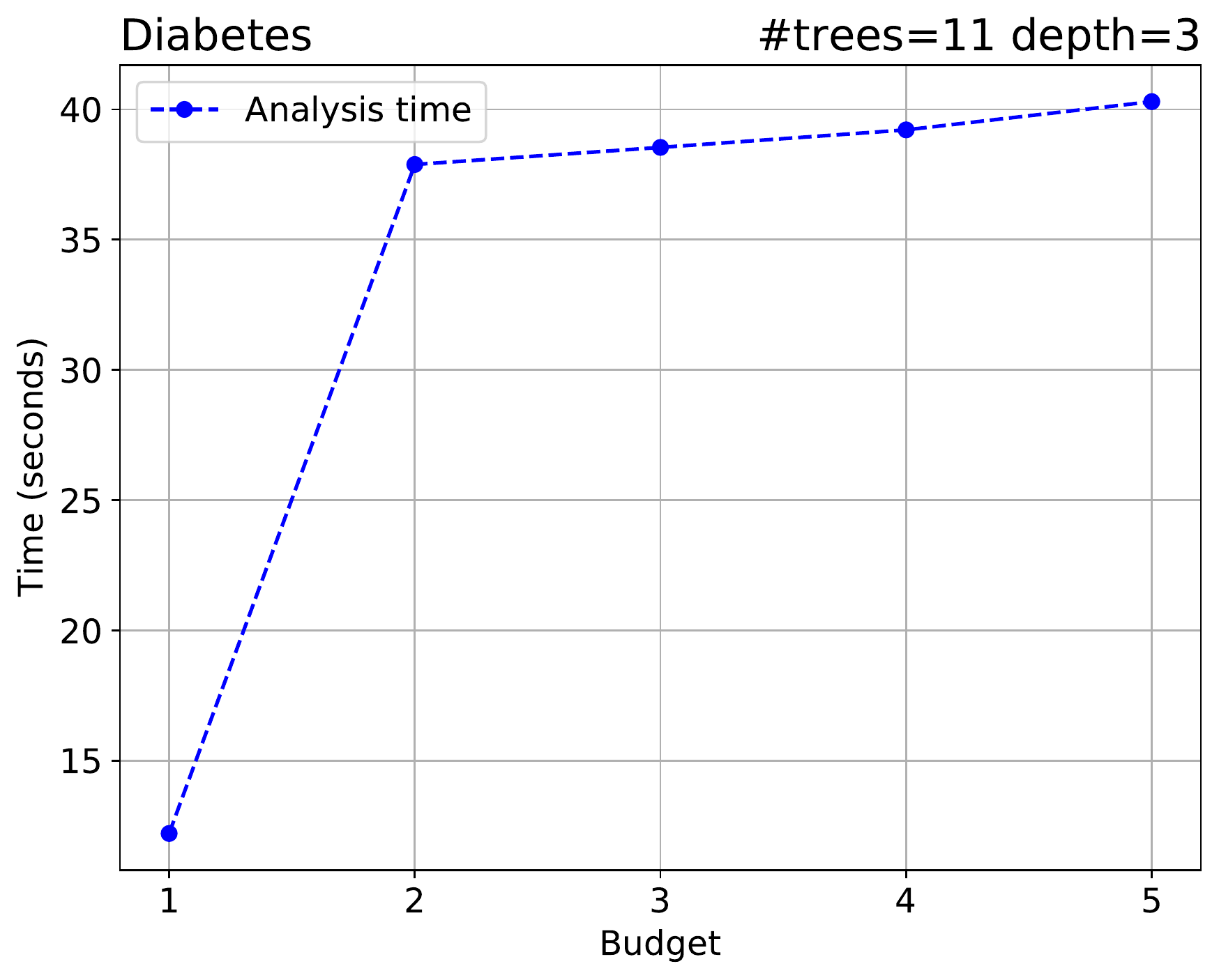}
    \caption{Analysis times when varying the attacker's budget}
    \label{fig:perf-ensembles-budget}
\end{figure}

\subsection{Discussion}
In the end, our experimental analysis yields positive results. We showed that resilience is \emph{useful}, because robustness may give a false sense of security, which is largely mitigated by the use of resilience. We also proved that our under-approximated resilience verification technique is \emph{precise}: we empirically showed that potentially large gaps between robustness and estimated resilience are motivated, because the estimated resilience is close to the robustness measured over the ``most unlucky'' sampling performed in a small neighborhood of the original test set. This confirms that our resilience estimates capture effective evasion attacks against plausible samplings of the same data distribution used to build the original test set.

Finally, we showed that resilience verification is \emph{feasible} in practice, at least for the relatively small models and simple datasets considered in the present work. For larger models, we showed that the iterative refinement process supported by our analysis technique can be leveraged to obtain useful under-approximations of resilience even before analysis convergence. Moreover, since the soundness proofs of our analysis abstract from several implementation details, e.g., the splitting criterion for symbolic attacks,  different heuristics may be tried out to further improve the analysis in terms of both precision and efficiency. We leave a more thorough investigation of this point to future work.

\section{Related Work}
\label{sec:related}

\subsection{Global Robustness}
\revise{Recent independent work in the area also acknowledged the limitations of robustness for the security verification of classifiers~\cite{Chen0QLJW21,LeinoWF21}. Chen \etal{} defined a set of new \emph{global robustness} properties, i.e., universally-quantified statements over one or more inputs to the classifier and its corresponding outputs~\cite{Chen0QLJW21}. They also formalized a data-independent stability definition, that requires any two inputs differing just for the value of a fixed set of features to lead to ``close'' predictions, and proposed a technique to verify this property for a custom type of rule-based classifiers generalizing decision tree ensembles. Although their work shares similarities with ours in terms of research goals, we note several important differences. First, our data-independent stability analysis allows one to identify a subset of the feature space where the classifier is stable, rather than verifying stability over the entire feature space. This is more useful in practice, because stability over the entire feature space is often too strong, essentially requiring that the set of non-robust features is unused for classification. Indeed, contrary to resilience, their global robustness properties entirely abstract from the data distribution, which is a sound yet overly conservative choice in the ML setting. This claim is confirmed by the experimental evaluation in~\cite{Chen0QLJW21}, which shows that stability cannot be verified for any model (standard or robust) besides those deliberately trained by the authors to enforce that property. Rather, we are able to use our resilience notion to perform practically useful security evaluations of existing ML models, while still overcoming the limitations of robustness confirmed by our experiments.}

\revise{Leino \etal, instead, introduced globally-robust neural networks~\cite{LeinoWF21}. They proposed a technique to train neural networks with a special output $\bot$, designed to signal predictions performed on a subset of the feature space that is too close to the decision boundary, hence potentially subject to evasion attacks. Their notion of global robustness requires that any two ``close'' inputs must lead to the same output, unless $\bot$ is returned for at least one of the two inputs. Our data-independent stability analysis essentially captures the same notion, because its output could also be used to return $\bot$ on all the instances that do not fall on a stable subset of the feature space (see Section~\ref{sec:res-vs-global}). However, observe that global robustness as defined in~\cite{LeinoWF21} cannot be used to reason about the security of traditional classifiers which do not use the $\bot$ label and their technical treatment is quite different from ours, because in this paper we operate on decision tree models rather than on neural networks.}

\subsection{Security of Decision Trees}
The security certification of decision trees and decision tree ensembles has received an increasing amount of attention by the research community during the last years. The first seminal work on the topic is due to Kantchelian, Tygar and Joseph~\cite{KantchelianTJ16}. They showed that computing minimal adversarial perturbations for tree ensembles is NP-complete in general, and proposed a mixed-integer linear programming technique for the task. This motivated additional work in the area by Chen \etal~\cite{ChenZS0BH19}. They investigated restricted fragments of the problem which are tractable in polynomial time and proposed an approximated, yet sound, approach to verify robustness against $L_\infty$-norm attackers. In later work, Ranzato and Zanella proposed the use of abstract interpretation to mitigate the complexity of robustness verification by means of a sound over-approximation of the ensemble predictions, again assuming $L_\infty$-norm attackers~\cite{RanzatoZ20}. Calzavara, Ferrara and Lucchese showed that abstract interpretation could also be used to verify the robustness of decision trees against an expressive threat model, where the attacker is encoded as an arbitrary imperative program~\cite{CalzavaraFL20}. All of these approaches only prove robustness and cannot be directly used to prove resilience without a data-independent stability analysis, like the one proposed in the present paper.

A different line of work which is more directly comparable to ours is related to the VoTE checker by T{\"{o}}rnblom and Nadjm{-}Tehrani~\cite{TornblomN20}. Given a tree ensemble, VoTE computes the set of all the equivalence classes induced by the ensemble over the feature space. Once the equivalence classes have been computed, VoTE uses a property checker module to verify different properties on them. The idea of computing equivalence classes from the tree ensemble yields a data-independent analysis approach, however there are important differences with respect to our work. First, their analysis is not adversary-aware and the security implications of data-independence are not explored by the authors, since they just verify traditional robustness properties. We rather clarify the practical relevance of data-independence by introducing a new formal security notion called resilience and we design experiments to show its empirical value on real datasets. Moreover, computing all the equivalence classes of an ensemble is infeasible in general due to their combinatorial explosion, as also observed by the authors of VoTE. In a follow-up work, the same authors proposed an abstraction-refinement approach to mitigate this complexity problem~\cite{TornblomN19}. However, contrary to our analysis, their extension is not proved sound, which is an important requirement for the analysis of classifiers deployed in adversarial settings.

Finally, we mention that several papers discussed new algorithms for training decision trees and decision tree ensembles that are robust to evasion attacks~\cite{CalzavaraLT19,CalzavaraLTAO20,ChenZBH19,Andriushchenko019,VosV21}. These works are complementary to our verification technique, which can be applied to both standard and robust trees, as we discussed in our experimental evaluation.

\section{Conclusion}
We criticized the traditional robustness measure used to assess the security of classifiers against evasion attacks and we proposed an improved measure called resilience, which provides additional assurances on unsampled data outside the test set. We then discussed how resilience can be estimated by combining traditional tools for robustness verification with a data-independent stability analysis, which does not depend on a specific test set. We finally proposed a formally sound data-independent stability analysis for decision trees and decision tree ensembles, which we evaluated on public datasets with positive results. By using our stability analysis, we managed to establish precise and practically useful estimates of resilience within a reasonable amount of time.

We see several interesting directions for future work. First, we would like to extend our analysis to gradient-boosted decision trees~\cite{Friedman01} and more sophisticated threat models, e.g., where adversarial manipulations are expressed in terms of Euclidean distances or via rewriting rules~\cite{CalzavaraFL20}. Moreover, we plan to design new training algorithms based on resilience minimization and a new boosting algorithm for robust tree ensembles based on our stability analysis. Indeed, since our analysis is able to identify the weak spots where the ensemble might be unstable, the boosting algorithm may train additional trees designed to provide stability on that part of the feature space. Finally, we would like to generalize our resilience verification approach to classifiers different than decision trees and decision tree ensembles, e.g., SVMs~\cite{BiggioNL11} and deep neural networks~\cite{SinghGMPV18}.

\bibliographystyle{plain}
\bibliography{main}

\appendix
\section{Proofs}
\label{sec:proofs}
We provide proofs of the formal results in the paper.

\subsection{Proof of Theorem~\ref{thm:tree}}
The proof leverages a key technical lemma (Lemma~\ref{lem:annotation}) formalizing the soundness of the tree annotation function in Algorithm~\ref{alg:annotate}. More specifically, we prove that the tree annotation function always produces a \emph{well-annotated} decision tree according to the following definition.

\begin{definition}[Well-Annotated Decision Tree]
\label{def:well-annotated}
The node $n$ of the decision tree $t$ is \emph{well-annotated} by the set of symbolic attacks $S$ if and only if, for every instance $\vec{x} \in \feats$ and every $\vec{z} \in A(\vec{x})$ such that $n$ is traversed in the prediction $t(\vec{z})$, $S$ contains an element $\symatk{\ipre_1,\ldots,\ipre_d}{\ipost_1,\ldots,\ipost_d}{k}$ such that $\vec{x} \in \langle \ipre_1,\ldots,\ipre_d \rangle$, $\vec{z} \in \langle \ipost_1,\ldots,\ipost_d \rangle$ and $k$ is the minimum cost to pay to make $\vec{x}$ traverse $n$. We say that the decision tree $t$ is well-annotated if and only if all its nodes are well-annotated by the set of symbolic attacks stored in their $\sym$ attribute.
\end{definition}

\begin{lemma}[Soundness of Tree Annotation]
\label{lem:annotation}
The call $\Call{Annotate}{t,S}$ returns a well-annotated decision tree, provided that the root of $t$ is well-annotated by $S$.
\end{lemma}
\begin{proof}
The proof is by induction on the depth of the tree $t$. If the tree has depth 1, then it includes a single node, i.e., the root, and the conclusion follows by the assumption that the root of $t$ is well-annotated by $S$. Otherwise, we have $t = \sigma(f,v,t_r,t_r)$ for some feature $f$, threshold $v$ and sub-trees $t_l,t_r$. The function then computes two new sets of symbolic attacks $S_l,S_r$ before invoking $\Call{Annotate}{t_l,S_l}$ and $\Call{Annotate}{t_r,S_r}$. Hence, the desired conclusion follows by inductive hypothesis, provided that we are able to show that the roots of $t_l$ and $t_r$ are well-annotated by $S_l$ and $S_r$ respectively. We just prove the former, since the latter uses an equivalent reasoning.

Pick any instance $\vec{x} \in \feats$ and consider any $\vec{z} \in A(x)$, we observe that $S$ must contain an element $s$ such that $\vec{x} \in s.\pre$, $\vec{z} \in s.\post$ and $s.\paid = 0$, because all instances must traverse the root. Assume $z_f \leq v$, we prove that $\Call{RefineLeft}{s,f,v}$ returns a set of symbolic attacks $S_l' \subseteq S_l$ such that there exists $s' \in S_l'$ such that $\vec{x} \in s'.\pre$, $\vec{z} \in s'.\post$ and $s'.\paid$ is the minimum cost to pay to make $\vec{x}$ traverse the left child of the root. Assume $s.\pre = \langle \ipre_1,\ldots,\ipre_d \rangle$, $s.\post = \langle \ipost_1,\ldots,\ipost_d \rangle$ and $\iatk_f = \langle \delta_l,\delta_r \rangle$, we discriminate four cases:
\begin{itemize}
    \item If $\ipre_f = \ipost_f$ and $x_f \leq v$, we leverage the observation that $\vec{z} \in s.\post$ and $z_f \leq v$, hence the condition $\ipost_f \cap (-\infty,v] \neq \emptyset$ at line 6 must be satisfied. In this case, $S_l'$ must contain an $s''$ such that:
    \begin{itemize}
        \item $s''.\pre = \langle \ipre_1,\ldots,\ipre_{f-1},\ipre_f \cap (-\infty,v],\ipre_{f+1},\ldots,\ipre_d \rangle$
        \item $s''.\post = \langle \ipost_1,\ldots,\ipost_{f-1},\ipost_f \cap (-\infty,v],\ipost_{f+1},\ldots,\ipost_d \rangle$
        \item $s''.\paid = 0$
    \end{itemize}
    The conclusion follows by the observation that $s''$ satisfies the three required conditions on $s'$.
    
    \item If $\ipre_f = \ipost_f$ and $x_f > v$, we leverage the observation that $\vec{z} \in A(x)$ and $z_f \leq v$. This implies that $\delta_l < 0$, $x_f \in (v,v-\delta_l]$ and $z_f \in (v+\delta_l,v]$; moreover, we must have $c_f \leq b$. By combining all this information and the observation that $s.\paid = 0$, we conclude that the condition at line 14 must be satisfied. In this case, $S_l'$ must contain an $s''$ such that:
    \begin{itemize}
        \item $s''.\pre = \langle \ipre_1,\ldots,\ipre_{f-1},\ipre_f \cap (v,v-\delta_l],\ipre_{f+1},\ldots,\ipre_d \rangle$
        \item $s''.\post = \langle \ipost_1,\ldots,\ipost_{f-1},\ipost_f \cap (v+\delta_l,v], \ipost_{f+1},\ldots,\ipost_d \rangle$
        \item $s''.\paid = c_f$
    \end{itemize}
    The conclusion follows by the observation that $s''$ satisfies the three required conditions on $s'$.
    
    \item If $\ipre_f \neq \ipost_f$ and $x_f \leq v$, we leverage the observation that $\vec{z} \in s.\post$ and $z_f \leq v$, hence the condition $\ipost_f \cap (-\infty,v] \neq \emptyset$ at line 6 must be satisfied. In this case, $S_l'$ must contain an $s''$ such that:
    \begin{itemize}
        \item $s''.\pre = \langle \ipre_1,\ldots,\ipre_{f-1},\ipre_f \cap (-\infty,v-\min(0,\delta_l)],\ipre_{f+1},\ldots,\ipre_d \rangle$
        \item $s''.\post = \langle \ipost_1,\ldots,\ipost_{f-1},\ipost_f \cap (-\infty,v],\ipost_{f+1},\ldots,\ipost_d \rangle$
        \item $s''.\paid = k$
    \end{itemize}
    The conclusion follows by the observation that $s''$ satisfies the three required conditions on $s'$.
    
    \item If $\ipre_f \neq \ipost_f$ and $x_f > v$, we leverage the observation that $\vec{z} \in A(x)$ and $z_f \leq v$. This implies that $\delta_l < 0$, $x_f \in (v,v-\delta_l]$ and $z_f \in (v+\delta_l,v]$. We then observe that $\vec{z} \in s.\post$ and $z_f \leq v$, hence the condition $\ipost_f \cap (-\infty,v] \neq \emptyset$ at line 6 must be satisfied. In this case, $S_l'$ must contain an $s''$ such that:
    \begin{itemize}
        \item $s''.\pre = \langle \ipre_1,\ldots,\ipre_{f-1},\ipre_f \cap (-\infty,v-\min(0,\delta_l)],\ipre_{f+1},\ldots,\ipre_d \rangle$
        \item $s''.\post = \langle \ipost_1,\ldots,\ipost_{f-1},\ipost_f \cap (-\infty,v],\ipost_{f+1},\ldots,\ipost_d \rangle$
        \item $s''.\paid = k$
    \end{itemize}
    The conclusion follows by the observation that $s''$ satisfies the three required conditions on $s'$.
\end{itemize}
\end{proof}

We now move back to the proof of the theorem. Consider an instance $\vec{x}$ and an adversarial perturbation $\vec{z} \in A(\vec{x})$ such that $t(\vec{z}) \neq t(\vec{x})$. This means that there exist two leaves $\lambda(y)$ and $\lambda'(y')$ with $y \neq y'$ such that $t(\vec{x}) = y$ and $t(\vec{z}) = y'$. By Lemma~\ref{lem:annotation}, $t$ must be well-annotated after line 2, hence we can make the following observations by Definition~\ref{def:well-annotated}:
\begin{enumerate}
    \item Since $\vec{x} \in A(\vec{x})$, the leaf $\lambda(y)$ must contain a symbolic attack $s = \symatk{\ipre_1,\ldots,\ipre_d}{\ipost_1,\ldots,\ipost_d}{k}$ such that $\vec{x} \in \langle \ipre_1,\ldots,\ipre_d \rangle$, $\vec{x} \in \langle \ipost_1,\ldots,\ipost_d \rangle$ and $k = 0$.
    \item Since $\vec{z} \in A(\vec{x})$, the leaf $\lambda'(y')$ must contain a symbolic attack $s' = \symatk{\jpre_1,\ldots,\jpre_d}{\jpost_1,\ldots,\jpost_d}{k'}$ such that $\vec{x} \in \langle \jpre_1,\ldots,\jpre_d \rangle$, $\vec{z} \in \langle \jpost_1,\ldots,\jpost_d \rangle$ and $k'$ is the minimum cost to pay to make $\vec{x}$ traverse $\lambda'(y')$. This cost must be greater than 0, because $t(\vec{z}) \neq t(\vec{x})$ implies $\vec{z} \neq \vec{x}$.
\end{enumerate}

This implies that line 9 is reachable and $s.\pre \cap s'.\pre \neq \emptyset$, hence a new symbolic attack $s''$ is added to the return value $U$ at lines 10-13. Thus, we just need to show that $s''$ satisfies the conditions of the theorem:
\begin{itemize}
    \item We have that $\vec{x} \in s''.\pre = s.\pre \cap s'.\pre$, by points 1 and 2.
    \item We have that $\vec{z} \in s'.\post$ by point 2. Moreover, since $\vec{z} \in A(\vec{x})$, we must have $\vec{z} \in \vec{x} + \langle \iatk_1,\ldots,\iatk_d \rangle$ by definition of adversarial manipulation. Since $\vec{x} \in s''.\pre$ by the previous point, we get $\vec{z} \in s''.\pre + \langle \iatk_1,\ldots,\iatk_d \rangle$, hence we conclude $\vec{z} \in s'.\post \cap (s''.\pre + \langle \iatk_1,\ldots,\iatk_d \rangle) = s''.\post$ as desired.
\end{itemize}

\subsection{Proof of Theorem~\ref{thm:ensemble}}
We prove the following invariant for the outer loop: for every instance $\vec{x} \in \feats$ and every adversarial manipulation $z \in A(\vec{x})$ such that $T(\vec{z}) \neq T(\vec{x})$, there exists $s \in C \cup E$ such that $\vec{x} \in s.\pre$ and $\vec{z} \in s.\post$.

Let $T = \{t_1,\ldots,t_n\}$, consider an instance $\vec{x}$ and an adversarial manipulation $\vec{z} \in A(\vec{x})$ such that $T(\vec{z}) \neq T(\vec{x})$. We first prove the base case, i.e., we show that the invariant holds when no loop iteration has taken place. Initially, $C = \cup_i U_i$ where each $U_i$ is computed by calling $\Call{Analyze}{t_i}$. Since $T(\vec{z}) \neq T(\vec{x})$, there exists $t_i \in T$ such that $t_i(\vec{z}) \neq t_i(\vec{x})$. Hence, there exists $s \in U_i$ such that $\vec{x} \in s.\pre$ and $\vec{z} \in s.\post$ by Theorem~\ref{thm:tree}. The conclusion then follows by definition of $C$.

Assume now the invariant holds up to a given iteration, we show it is preserved at the next iteration. By inductive hypothesis there exists $s \in C \cup E$ such that $\vec{x} \in s.\pre$ and $\vec{z} \in s.\post$. We distinguish two cases. If $s \in E$, then the conclusion is immediate because nothing is ever removed from $E$. If instead $s \in C$, we show that each iteration of the inner loop cannot break the outer loop invariant. In particular, assume some $s' \in C$ is processed by an iteration of the inner loop, leading to updated $C'$ and $E'$ respectively. We can distinguish the following cases at the end of the iteration:
\begin{itemize}
    \item If $C' = C \setminus \{s'\}$ and $E' = E$, then there exists $y$ such that $T(s.\pre) = T(s.\post) = \{y\}$. This implies that for all instances $\vec{w} \in s.\pre \cup s.\post$ we have $T(\vec{w}) = y$, thanks to the first soundness condition. Since $T(\vec{x}) \neq T(\vec{z})$, we have that either $\vec{x} \not\in s.\pre$ or $\vec{z} \not\in s.\post$, hence $s' \neq s$ and the loop invariant is preserved.
    
    \item If $C' = (C \setminus \{s'\}) \cup \Call{Split}{s'}$ and $E' = E$, the loop invariant is preserved by the second soundness condition.
    
    \item if $C' = C \setminus \{s'\}$ and $E' = E \cup \{s'\}$, then $C' \cup E' = C \cup E$ and thus the loop invariant is preserved.
\end{itemize}

%\newpage
%\input{appendix-algo}

\end{document}